%% file: main.tex
\documentclass[10pt,journal,compsoc]{IEEEtran}
\usepackage{epsfig}
\usepackage{graphicx}
\usepackage{amsmath}
\usepackage{amssymb}
\usepackage{bm}
\usepackage{color}
\usepackage{soul}
\usepackage{algorithm}
\usepackage{algorithmic}
\usepackage{comment}
\usepackage{subfigure}
\usepackage{multirow}
\usepackage{float}
\usepackage{enumitem}
\usepackage{capt-of}
\usepackage{url}

\ifCLASSOPTIONcompsoc
  \usepackage[nocompress]{cite}
\else
  \usepackage{cite}
\fi
\newcommand{\cI}{\mathcal{I}}
\newcommand{\bx}{\mathbf{x}}
\newcommand{\bz}{\mathbf{z}}

\newcommand{\bs}{\mathbf{s}}
\newcommand{\br}{\mathbf{r}}
\newcommand{\bu}{\mathbf{u}}

\newcommand{\bq}{\mathbf{q}}
\newcommand{\bv}{\mathbf{v}}
\newcommand{\ba}{\mathbf{a}}

\newcommand{\bE}{\mathbf{E}}
\newcommand{\bg}{\mathbf{g}}
\newcommand{\bh}{\mathbf{h}}

\newcommand{\bM}{\mathbf{M}}

\newcommand{\bP}{\mathbf{P}}

\newcommand{\bB}{\mathbf{B}}
\newcommand{\bC}{\mathbf{C}}

\newcommand{\cH}{\mathcal{H}}

\newcommand{\cP}{\mathcal{P}}
\newcommand{\cG}{\mathcal{G}}

\newcommand{\cL}{\mathcal{L}}
\newcommand{\bI}{\mathbf{I}}

\newcommand{\btheta}{\bm{\theta}}
\newcommand{\blambda}{\boldsymbol{\lambda}}

\newcommand{\bone}{\mathbf{1}}
\newcommand{\bZero}{\mathbf{0}}
\newcommand{\bA}{\mathbf{A}}

\newcommand{\bG}{\mathbf{G}}
\newcommand{\bb}{\mathbf{b}}
\newcommand{\bc}{\mathbf{c}}
\newcommand{\bbR}{\mathbb{R}}
\newcommand{\bbI}{\mathbb{I}}

\newcommand{\ep}{EP}
\newcommand{\admm}{AM}

\newcommand{\qed}{\nobreak \ifvmode \relax \else
	\ifdim\lastskip<1.5em \hskip-\lastskip
	\hskip1.5em plus0em minus0.5em \fi \nobreak
	\vrule height0.75em width0.5em depth0.25em\fi}
\newtheorem{theorem}{Theorem}

\newenvironment{proof}[1][Proof]{\begin{trivlist}
		\item[\hskip \labelsep {\it #1}]}{\end{trivlist}}

\DeclareMathOperator*{\argmin}{arg\,min}
\newcommand{\stackedVector}[2]
{ 	\begin{bmatrix}
		#1 \\ #2	
	\end{bmatrix}
}

\newcommand{\uiprime}{u_i^\prime}
\newcommand{\siprime}{s_i^\prime}
\newcommand{\itt}[1]{{#1}^{(t)}}

\begin{document}
%
% paper title
% Titles are generally capitalized except for words such as a, an, and, as,
% at, but, by, for, in, nor, of, on, or, the, to and up, which are usually
% not capitalized unless they are the first or last word of the title.
% Linebreaks \\ can be used within to get better formatting as desired.
% Do not put math or special symbols in the title.
\title{Deterministic Approximate Methods for\\Maximum Consensus Robust Fitting}

\author{Huu Le,
  Tat-Jun Chin, %~\IEEEmembership{Member,~IEEE,}
  Anders Eriksson, 
  Thanh-Toan Do,
  and~David Suter% <-this % stops a space
  \IEEEcompsocitemizethanks{\IEEEcompsocthanksitem H. Le, T.-J. Chin, and D.Suter are with the school of Computer Science, The University of Adelaide	\protect\\
    % note need leading \protect in front of \\ to get a newline within \thanks as
    % \\ is fragile and will error, could use \hfil\break instead.
   E-mail: \{huu.le,tat-jun.chin, david.suter\}@adelaide.edu.au \protect\\   
    A. Eriksson is with the school of Computer Science and Electrical Engineering, Queensland University of Technology	\protect\\
    E-mail: anders.eriksson@qut.edu.au 
  }% <-this % stops an unwanted space
  \thanks{Manuscript received }
}

\IEEEtitleabstractindextext{%
\begin{abstract}
Maximum consensus estimation plays a critically important role in robust fitting problems in computer vision. Currently, the most prevalent algorithms for consensus maximization draw from the class of randomized hypothesize-and-verify algorithms, which are cheap but can usually deliver only rough approximate solutions. On the other extreme, there are exact algorithms which are exhaustive search in nature and can be costly for practical-sized inputs. This paper fills the gap between the two extremes by proposing \emph{deterministic} algorithms to \emph{approximately} optimize the maximum consensus criterion. Our work begins by reformulating consensus maximization with linear complementarity constraints. Then, we develop two novel algorithms: one based on non-smooth penalty method with a Frank-Wolfe style optimization scheme, the other based on the Alternating Direction Method of Multipliers (ADMM). Both algorithms solve convex subproblems to efficiently perform the optimization. We demonstrate the capability of our algorithms to greatly improve a rough initial estimate, such as those obtained using least squares or a randomized algorithm. Compared to the exact algorithms, our approach is much more practical on realistic input sizes. Further, our approach is naturally applicable to estimation problems with geometric residuals. \emph{Matlab code and demo program for our methods can be downloaded from https://goo.gl/FQcxpi.}
\end{abstract}

% Note that keywords are not normally used for peerreview papers.
}

% make the title area
\maketitle

% To allow for easy dual compilation without having to reenter the
% abstract/keywords data, the \IEEEtitleabstractindextext text will
% not be used in maketitle, but will appear (i.e., to be "transported")
% here as \IEEEdisplaynontitleabstractindextext when the compsoc 
% or transmag modes are not selected <OR> if conference mode is selected 
% - because all conference papers position the abstract like regular
% papers do.
\IEEEdisplaynontitleabstractindextext
% \IEEEdisplaynontitleabstractindextext has no effect when using
% compsoc or transmag under a non-conference mode.

% For peer review papers, you can put extra information on the cover
% page as needed:
% \ifCLASSOPTIONpeerreview
% \begin{center} \bfseries EDICS Category: 3-BBND \end{center}
% \fi
%
% For peerreview papers, this IEEEtran command inserts a page break and
% creates the second title. It will be ignored for other modes.
\IEEEpeerreviewmaketitle

\IEEEraisesectionheading{\section{Introduction}\label{sec:introduction}}
% Computer Society journal (but not conference!) papers do something unusual
% with the very first section heading (almost always called "Introduction").
% They place it ABOVE the main text! IEEEtran.cls does not automatically do
% this for you, but you can achieve this effect with the provided
% \IEEEraisesectionheading{} command. Note the need to keep any \label that
% is to refer to the section immediately after \section in the above as
% \IEEEraisesectionheading puts \section within a raised box.

% The very first letter is a 2 line initial drop letter followed
% by the rest of the first word in caps (small caps for compsoc).
% 
% form to use if the first word consists of a single letter:
% \IEEEPARstart{A}{demo} file is ....
% 
% form to use if you need the single drop letter followed by
% normal text (unknown if ever used by the IEEE):
% \IEEEPARstart{A}{}demo file is ....
% 
% Some journals put the first two words in caps:
% \IEEEPARstart{T}{his demo} file is ....
% 
% Here we have the typical use of a "T" for an initial drop letter
% and "HIS" in caps to complete the first word.
\IEEEPARstart{R}{obust} model fitting lies at the core of computer vision, due to the need of many fundamental tasks to deal with real-life data that are noisy and contaminated with outliers. To conduct robust model fitting, a robust fitting criterion is optimized w.r.t.~a set of input measurements. Arguably the most popular robust criterion is \emph{maximum consensus}, which aims to find the model that is consistent with the largest number of inliers, i.e., has the highest consensus.

Due to the critical importance of maximum consensus estimation, considerable effort has been put into devising algorithms for optimizing the criterion. A large amount of work occurred within the framework of hypothesize-and-verify methods, i.e., RANSAC~\cite{fischler1981random} and its variants. Broadly speaking, these methods operate by fitting the model onto randomly sampled minimal subsets of the data, and returning the candidate with the largest inlier set. Improvements to the basic algorithm include guided sampling and speeding up the model verification step~\cite{choi09}.

An important innovation is locally optimized RANSAC (LO-RANSAC)~\cite{chum2003locally,lebeda2012fixing}. As the name suggests, the objective of the method is to locally optimize RANSAC estimates. This is achieved by embedding in RANSAC an inner hypothesize-and-verify routine, which is triggered whenever the solution is updated in the outer loop. Different from the main RANSAC algorithm, the inner subroutine generates hypotheses from larger-than-minimal subsets sampled from the inlier set of the incumbent solution, in the hope of driving it towards an improved result.

Though efficient, there are fundamental shortcomings in the hypothesize-and-verify heuristic. Primarily, its randomized nature does not provide an absolute certainty whether the obtained result is a satisfactory approximation. Moreover, when data is contaminated with a high proportion of outliers, such randomized methods tend to be computationally expensive, because the probability of randomly picking a clean minimal subset decreases exponentially with the number of outliers. In LO-RANSAC, this weakness also manifests in the inner RANSAC routine, in that it is essentially a randomized trial-and-error procedure instead of a directed search to improve the estimate.

%its randomized nature does not guarantee finding a good maximum consensus estimate; different runs may also give unpredictably different results. 

%The randomized nature of the heuristic also means that different runs may give unpredictably different results, which makes it non-ideal for tasks that require high repeatability.

More recently, there is a growing number of globally optimal algorithms for consensus maximization~\cite{olsson07,zheng09,enqvist12,li14,chin15}. The fundamental intractability of maximum consensus estimation, however, means that the global optimum can only be found by searching. The previous techniques respectively conduct branch-and-bound search~\cite{zheng09,li14}, tree search~\cite{chin15}, or enumeration~\cite{olsson07,enqvist12}. Thus, global algorithms are practical only on problems with a small number of measurements and/or models of low dimensionality.

So far, what is sorely missing in the literature is an algorithm that lies in the middle ground between the above two extremes. Specifically, a maximum consensus algorithm that is \emph{approximate} and \emph{deterministic}, would add significantly to the robust fitting toolbox of computer vision.

In this paper, we contribute two such algorithms. Our starting point is to reformulate consensus maximization with linear complementarity constraints. We then develop an algorithm based on non-smooth penalty method supported by a Frank-Wolfe-style optimization scheme, and another algorithm based on the ADMM. In both algorithms, the calculation of the update step involves executing convex subproblems, which are efficient and enable the algorithms to handle realistic input sizes (hundreds to thousands of measurements). Further, our algorithms are naturally capable of handling the non-linear geometric residuals commonly used in computer vision~\cite{kahl05,ke05}.

As will be demonstrated experimentally, our algorithms can significantly improve rough estimates obtained using an initial method, such as least squares or a fast randomized scheme such as RANSAC. Qualitative improvements achieved by our algorithms are also greater than that of LO-RANSAC, while incurring only marginally higher runtimes.

\subsection{Deterministic robust fitting }\label{sec:irls}
M-estimators~\cite{huber81} are an established class of robust statistical methods. The M-estimate is obtained by minimizing the sum of a set of $\rho$ functions defined over the residuals, where $\rho$ (e.g., the Huber norm) is responsible for discounting the effects of outliers. The primary technique for the minimization is iteratively reweighted least squares (IRLS). At each iteration, a weighted least squares problem is solved, where the weights are computed based on the previous estimate. Provided that $\rho$ satisfies certain properties~\cite{zhang97,aftab15}, IRLS will deterministically reduce the cost until a local minimum is reached. This however precludes consensus maximization, since the corresponding $\rho$ (a symmetric step function) is not positive definite and differentiable. Sec.~\ref{sec:char} will further explore the characteristics of the maximum consensus objective.

Arguably, one can simply choose a robust $\rho$ that works with IRLS and dispense with maximum consensus. However, another vital  requirement for IRLS to be feasible is that the weighted least squares problem is efficiently solvable. This unfortunately is not the case for many of the geometric distances used in computer vision~\cite{kahl05,ke05,hartley-accv07}.

The above limitations with IRLS suggest that deterministic approximate methods for robust fitting remain an open problem, and our proposed algorithms should represent significant contributions towards this direction.

\subsection{Road map}

The paper is structured as follows:
\begin{itemize}
\item Sec.~\ref{sec:problem_definition} defines the maximum consensus problem and characterizes the solution. It then describes the crucial reformulation with complementarity constraints.
\item Sec.~\ref{sec:penalty_method} describes the non-smooth penalty method.
\item Sec.~\ref{sec:admm} describes the ADMM-based algorithm.
\item Sec.~\ref{sec:quasiconvex} shows the applicability of our methods to estimation problems with quasiconvex geometric residuals.
\item Sec. \ref{sec:results} demonstrates the effectiveness of our methods through a set of experiments with synthetic and real data on common computer vision applications.
\end{itemize}
This paper is an extension of the conference version~\cite{le2017exact}, which proposed only the method based on non-smooth penalization. Sec.~\ref{sec:results} of the present paper experimentally compares the new ADMM technique with the penalty method.

\section{Problem definition}\label{sec:problem_definition}

We develop our algorithms in the context of fitting linear models, before extending to models with geometric residuals in Sec.~\ref{sec:quasiconvex}. Given a set of $N$ measurements $\{\bx_j, y_j\}_{j=1}^N$ for the linear model parametrized by vector $\btheta \in \bbR^d$, the goal of maximum consensus is to find the $\btheta$ that is consistent with as many of the input data as possible, i.e.,
\begin{align}\label{maxcon}
\begin{aligned}
\max_{\btheta \in \mathbb{R}^d}  && \Psi(\btheta)
\end{aligned}
\end{align}
where the objective function
\begin{align}\label{consensus}
\Psi(\btheta) = \sum_{j=1}^N \mathbb{I}\left(|\bx_j^T\btheta - y_j| \le \epsilon\right)
\end{align}
is the \emph{consensus} of $\btheta$. Here, $\mathbb{I}$ is the indicator function, which returns $1$ if its input predicate is true, and $0$ otherwise. The inlier-outlier dichotomy is achieved by comparing a residual $|\bx_j^T\btheta - y_j|$ with the pre-defined threshold $\epsilon$.

Expressing each inequality of the form $|\bx_j^T\btheta - y_j| \le \epsilon$ equivalently using the two linear constraints
\begin{align}\label{symmetric}
\bx_j^T\btheta - y_j \le \epsilon, \;\;\;\; -\bx_j^T\btheta + y_j \le \epsilon,
\end{align}
and collecting the data into the matrices
\begin{align}\label{data}
\begin{aligned}
\bA &= \left[ \begin{matrix} \bx_1, -\bx_1, \dots, \bx_N, -\bx_N \end{matrix}  \right], \\
\bb &= \left[ \begin{matrix} \epsilon+y_1, \epsilon-y_1, \dots, \epsilon+y_N, \epsilon-y_N \end{matrix}  \right]^T,
\end{aligned}
\end{align}
where $\bA \in \mathbb{R}^{d \times M}$, $\bb \in \mathbb{R}^{M}$ and $M = 2N$, we can redefine consensus as
\begin{align}\label{consensus_maxfs}
\Psi(\btheta) = \sum_{i=1}^M \mathbb{I}\left(\ba_i^T\btheta - b_i \le 0\right),
\end{align}
where $\ba_i$ is the $i$-th column of $\bA$ and $b_i$ is the $i$-th element of $\bb$. Plugging~\eqref{consensus_maxfs} instead of~\eqref{consensus} into~\eqref{maxcon} yields an equivalent optimization problem, in the sense that both objective functions have the same maximizers.

Henceforth, we will be developing our maximum consensus algorithm based on~\eqref{consensus_maxfs} as the definition of consensus.

\subsection{Characterizing the solution}\label{sec:char}

What does $\Psi$ look like? Consider the problem of robustly fitting a line onto a set of points $\{p_j,q_j\}^{N}_{j=1}$ on the plane. To apply formulation~\eqref{maxcon}, set $\bx_j = \left[\begin{matrix} p_j & 1 \end{matrix}\right]^T$ and $y_j = q_j$. The vector $\btheta \in \mathbb{R}^2$ then corresponds to the slope and intercept of the line. Fig.~\ref{objective} plots $\Psi(\btheta)$ for a sample point set $\{p_j,q_j\}^{N}_{j=1}$. As can be readily appreciated, $\Psi$ is a piece-wise constant step function, owing to the thresholding and discrete counting operations in the calculation of consensus.

\begin{figure}[ht]\centering
\subfigure[]{\includegraphics[width=0.8\columnwidth]{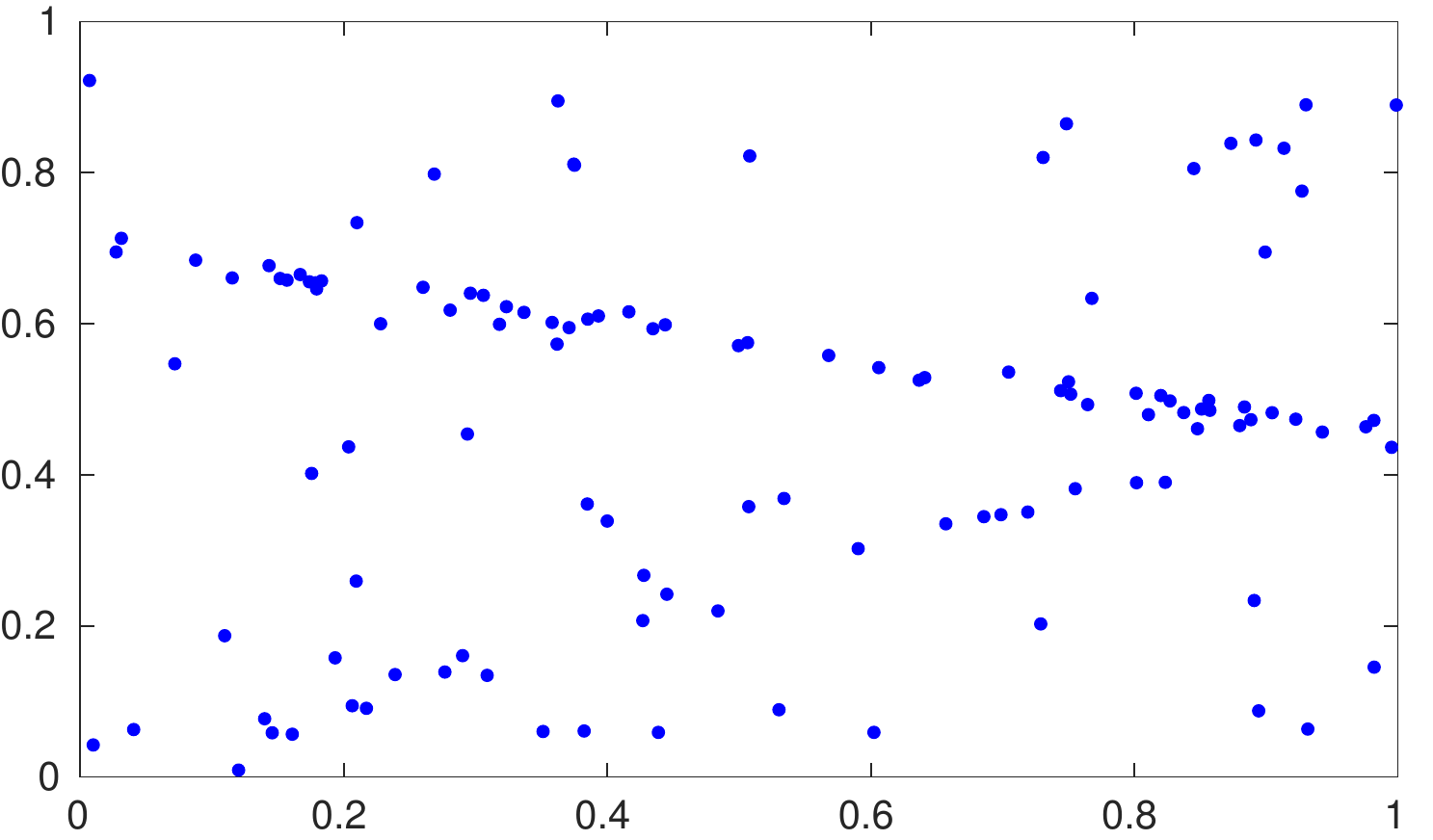}}
\subfigure[]{\includegraphics[width=0.8\columnwidth]{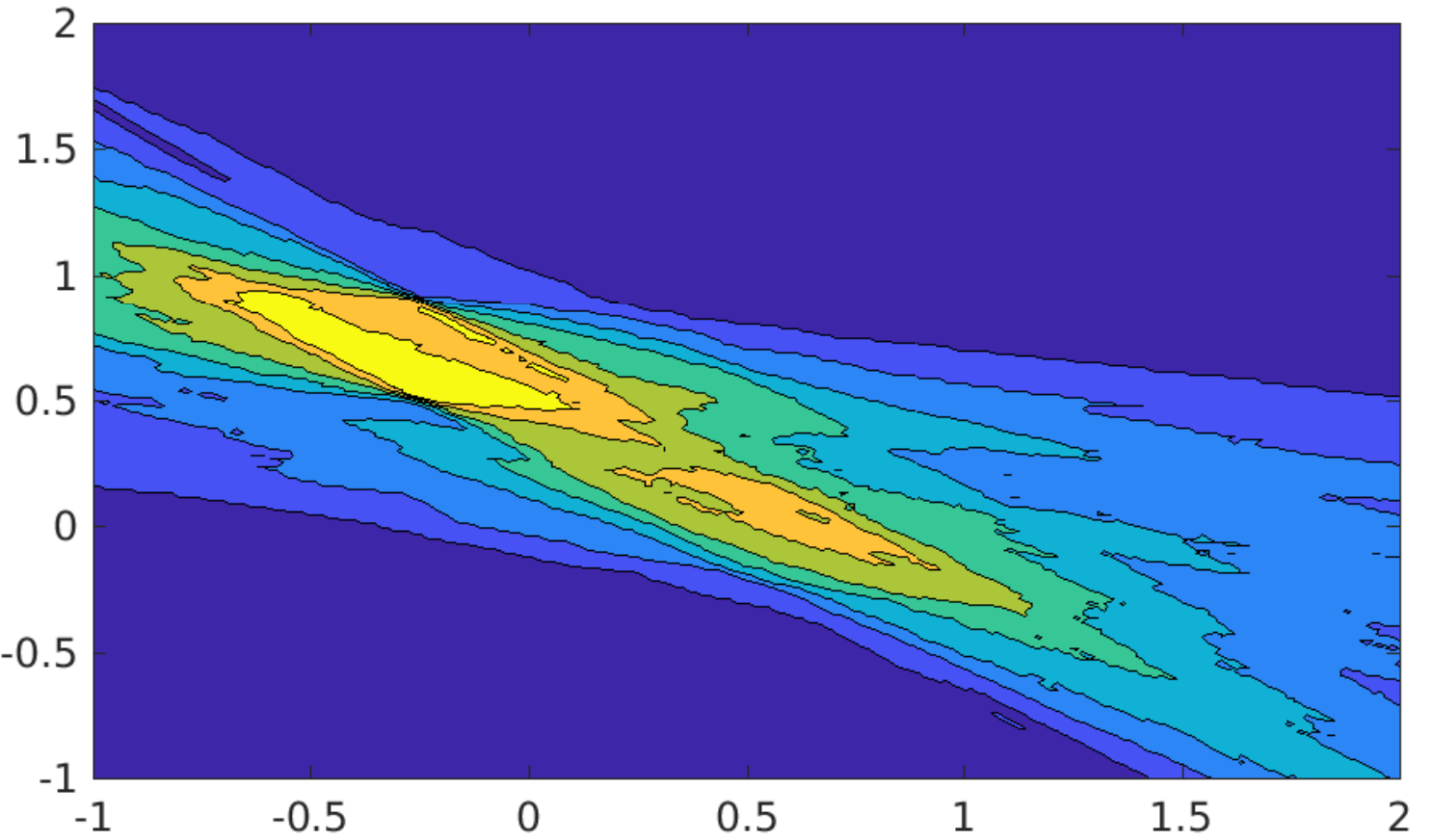}}
\caption{(a) Sample point set $\{p_j,q_j\}^{N}_{j=1}$. (b) A plot of $\Psi(\btheta)$ in $\mathbb{R}^2$ based on the sample point set. Each unique color represents a specific consensus value. Regions corresponding to the maximum consensus value are indicated in yellow.}
\label{objective}
\end{figure}

We define the \emph{global} or \emph{exact solution} to~\eqref{maxcon} as the vector $\btheta^\ast$ such that $\Psi(\btheta^\ast) \ge \Psi(\btheta)$ for all $\btheta \in \mathbb{R}^d$. In general, $\btheta^\ast$ is not unique, and can only be identified by searching. Recall that a \emph{local solution} of an unconstrained optimization problem
\begin{align}\label{uncons}
\begin{aligned}
\max_{\btheta \in \mathbb{R}^d}  && f(\btheta)
\end{aligned}
\end{align}
is a vector $\hat{\btheta}$ such that there exists a neighborhood $\mathcal{N} \subset \mathbb{R}^d$ of $\hat{\btheta}$ where $f(\hat{\btheta}) \ge f(\btheta)$ for all $\btheta \in \mathcal{N}$~\cite[Chap.~2]{nocedal}. By this definition, since $\Psi$ is piece-wise constant, all $\btheta \in \mathbb{R}^d$ are local solutions to~\eqref{maxcon}. The concept of \emph{local optimality} is thus not meaningful in the context of consensus maximization. Indeed, the lack of gradient information in $\Psi$ complicates the usage of standard nonlinear optimization schemes, which strive for local optimality, on problem~\eqref{maxcon} (cf.~IRLS).

Unlike nonlinear optmization methods or IRLS, the proposed algorithms do not depend on the existence of gradients; instead, our algorithms solve derived convex subproblems to deterministically and efficiently update an approximate solution to the maximum consensus problem. As mentioned in the introduction, such techniques have not been considered previously in the literature.

\subsection{Reformulation with complementarity constraints}\label{def_complem}

Introducing indicator variables $\bu \in \{0,1\}^M$ and slack variables $\bs \in \mathbb{R}^M$, we first reformulate~\eqref{maxcon} equivalently as an outlier count minimization problem
\begin{subequations}\label{minout}
	\begin{align}
	& \min_{\bu \in \{0,1\}^M, \; \bs \in \bbR^M, \; \btheta \in \bbR^d} &&\sum_{i}u_i \\
	& \text{subject to} && s_i - \ba_i^T\btheta + b_i \ge 0,\\
	& && u_i(s_i - \ba_i^T\btheta + b_i ) = 0, \label{complem1} \\
	& && s_i(1-u_i) = 0, \label{complem2} \\ 
	& && s_i \ge 0.
	\end{align}
\end{subequations}
Intuitively, $s_i$ must be non-zero if the $i$-th datum is an outlier w.r.t.~$\btheta$; in this case, $u_i$ must be set to $1$ to satisfy~\eqref{complem2}. In turn,~\eqref{complem1} forces the quantity $(s_i - \ba_i^T\btheta + b_i)$ to be zero. Conversely, if the $i$-th datum is an inlier w.r.t.~$\btheta$, then $s_i$ is zero, $u_i$ is zero and $(s_i - \ba_i^T\btheta + b_i)$ is non-zero. Observe, therefore, that~\eqref{complem1} and~\eqref{complem2} implement \emph{complementarity} between $u_i$ and $(s_i - \ba_i^T\btheta + b_i)$.

Note also that, due to the objective function and condition~\eqref{complem2}, the indicator variables can be relaxed without impacting the optimum, leading to the equivalent problem
\begin{subequations}\label{minout2}
	\begin{align}
	& \min_{\bu, \bs \in \bbR^M, \; \btheta \in \bbR^d} &&\sum_{i}u_i \\
	& \text{subject to} && s_i - \ba_i^T\btheta + b_i  \ge 0,\\
	& && u_i(s_i - \ba_i^T\btheta + b_i) = 0, \label{complem3} \\
	& && s_i(1-u_i) = 0, \label{complem4} \\ 
	& && 1 - u_i \ge 0, \\
	& && s_i, u_i \ge 0.
	\end{align}
\end{subequations}
This, however, does not make~\eqref{minout2} tractable, since~\eqref{complem3} and~\eqref{complem4} are bilinear in the unknowns.

To re-express~\eqref{minout2} using only positive variables, define
\begin{align}
\bv = \stackedVector{\btheta + \gamma\bone}{\gamma}, \;\;\;\; \bc_i = \left[ \begin{matrix} \ba_i^T & -\ba^T_i\bone \end{matrix} \right]^T,
\end{align}
where both are real vectors of length $(d+1)$. Problem~\eqref{minout2} can then be reformulated equivalently as
\begin{align}\label{minout3}
\begin{aligned}
& \min_{\bu, \bs \in \bbR^M, \; \bv \in \bbR^{d+1}} &&\sum_{i}u_i \\
& \text{subject to} && s_i - \bc_i^T\bv + b_i  \ge 0,\\
& && u_i(s_i - \bc_i^T\bv + b_i) = 0,\\
& && s_i(1-u_i) = 0,\\ 
& && 1 - u_i \ge 0, \\
& && s_i, u_i, v_i \ge 0.
\end{aligned}
\end{align}
Given a solution $\hat{\bu}$, $\hat{\bs}$ and $\hat{\bv}$ to~\eqref{minout3}, the corresponding solution $\hat{\btheta}$ to~\eqref{minout2} can be obtained by simply subtracting the last element of $\hat{\bv}$ from its first-$d$ elements.

While the relaxation does not change the fundamental intractability of~\eqref{maxcon}, that all the variables are now continuous allows to bring a broader class of optimization techniques to bear on the problem---as we will show next.

%\textcolor{red}{Note that by converting the problem into the form of~\eqref{minout3}, the discrete variables $u_i \in \{0, 1\}$ can be relaxed to $u_i \in [0,1]$ without changing the optimal solution. This advocates the use of optimization techniques that acts on a continuous search domain to solve~\eqref{minout3} regardless of the non-smoothness of the objective function, which will be shown later in our proposed methods. In addition, for each method, an analytical proof of convergence is provided and we empirically demonstrate that they can consistently achieve better solution quality than LO-RANSAC and other approximation or heuristics strategies.}

\section{Non-smooth penalty method}\label{sec:penalty_method}

Our first deterministic refinement algorithm is based on the technique of non-smooth penalization~\cite[Sec.~17.2]{nocedal}. Incorporating the equality constraints in~\eqref{minout3} into the cost function as a penalty term, we obtain the penalty problem
\begin{equation}\label{penalized}
\begin{aligned}
& \min_{\bu, \bs \in \bbR^M, \bv \in \bbR^{d+1} } && \sum_{i}u_i + \alpha \left[ u_i(s_i - \bc^T_i\bv + b_i ) + s_i(1-u_i) \right]\\
& \text{s.t.} && s_i - \bc_i^T\bv + b_i  \ge 0,\\
& && 1 - u_i \ge 0,\\
& && s_i,u_i,v_i \ge 0.
\end{aligned}
\end{equation}
The constant $\alpha \ge 0$ is called the \emph{penalty parameter}. Intuitively, the penalty term discourages solutions that violate the complementarity constraints, and the strength of the penalization is controlled by $\alpha$. Observe also that the remaining constraints in~\eqref{penalized} define a convex domain.

Henceforth, to reduce clutter, we sometimes use
\begin{align}
\bz = \left[ \bu^T~\bs^T~\bv^T \right]^T.
\end{align}
The cost function in~\eqref{penalized} can be rewritten as
\begin{align}\label{penalized2}
P(\bz \mid \alpha) = F(\bz) + \alpha Q(\bz),
\end{align}
where $F(\bz) = \| \bu \|_1$ and
\begin{align}
Q(\bz) &= \sum_{i} u_i(s_i - \bc^T_i\bv + b_i ) + s_i(1-u_i)\\
&= \sum_i s_i - u_i(\bc_i^T \bv - b_i).
\end{align}
Note that each summand in $Q(\bz)$ is non-negative, and the penalty term can be viewed as the $\ell_1$-norm (a non-smooth function) of the \emph{complementarity residual} vector
\begin{align}
\mathbf{r}(\bz) = \left[\begin{matrix} r_1(\bz) & \dots & r_M(\bz) \end{matrix} \right]^T,
\end{align}
where
\begin{align}
r_i(\bz) = s_i - u_i(\bc_i^T \bv - b_i).
\end{align}
In Sec.~\ref{sec:main}, we will devise a consensus maximization algorithm based on solving a sequence of the penalty problem~\eqref{penalized} with increasing values of $\alpha$. Before that, in Sec.~\ref{sec:solving_penalty}, we will discuss a method to solve the penalty problem for a given (constant) $\alpha$.

\subsection{Solving the penalty problem}\label{sec:solving_penalty}

\subsubsection{Necessary optimality conditions}\label{sec:stationarity}

Although $P(\bz \mid \alpha)$ is quadratic, problem~\eqref{penalized} is non-convex. However, it can be shown that~\eqref{penalized} has a vertex solution, i.e., a solution that is an extreme point of the convex set
\begin{align}
\begin{aligned}
\cP = \{ \bz \in \mathbb{R}^{2M+d+1} \mid &s_i - \bc_i^T\bv + b_i  \ge 0,\\
&1 - u_i \ge 0,\\
&s_i,u_i,v_i \ge 0,\\
&i=1,\dots,M\}
\end{aligned}
\end{align}
To minimize clutter, rewrite
\begin{align}
\cP = \{ \bz \in \mathbb{R}^{2M+d+1} \mid \bM\bz + \bq \ge \bZero, \;\; \bz \ge \bZero\},
\end{align}
where
\begin{align}
\begin{aligned}
\bM &= \left[ \begin{matrix}
\bZero & \bI & -\bC  \\	
-\bI & \bZero & \bZero  \\	
\end{matrix} \right],\\
\bC &= \left[ \begin{matrix} \bc_1 & \bc_2 & \dots & \bc_M \end{matrix} \right]^T,\\
\bq &= \left[ \begin{matrix} \bb^T & \bone^T \end{matrix} \right]^T;
\end{aligned}
\end{align}
(we do not define the sizes of $\bI$, $\bZero$ and $\bone$, but the sizes can be worked out based on the context). To begin, observe that the minima of~\eqref{penalized} obey the KKT conditions~\cite[Chap.~12]{nocedal}
\begin{align}\label{finalKKT}
\begin{aligned}
&\bu^{T}(-\alpha\bC\bv + \alpha\bb + \bone + \bm{\lambda}^\cG) = 0, \\
&\bs^T(\alpha\bone - \bm{\lambda}^\cH) = 0, \\
&\bv^T(-\alpha\bC^T\bu + \bC^T \bm{\lambda}^\cH)= 0, \\
&(\bm{\lambda}^\cH)^T (\bs - \bC \bv + \bb) = 0, \\
&(\bm{\lambda}^\cG)^T (\bone-\bu) = 0, \\
& \bs - \bC \bv + \bb \ge \bZero, \\
& \bone - \bu \ge \bZero,\\
&\bm{\lambda}^\cH, \bm{\lambda}^\cG, \bu, \bv, \bs \ge \bZero, \\
\end{aligned}
\end{align}
where $\bm{\lambda}^\cH = [ \lambda_1^\cH \; \dots \; \lambda_M^\cH ]^T$ and
$\bm{\lambda}^\cG = [ \lambda_1^\cG \; \dots \; \lambda_M^\cG ]^T$ are the Lagrange multipliers for the first two types of constraints in~\eqref{penalized}; {see the supplementary material (Section 2) for details}.

By rearranging, the KKT conditions~\eqref{finalKKT} can be summarized by the following relations
\begin{align}\label{lcp}
\bM^\prime\bz^\prime + \bq^\prime \ge \mathbf{0}, \;\; \bz^\prime \ge \mathbf{0}, \;\; (\bz^\prime)^T(\bM^\prime\bz^\prime + \bq^\prime) = 0,
\end{align}
where
\begin{align}\label{bigM}
\begin{aligned}
\bz^\prime &= \left[ \begin{matrix} \bz^T & (\bm{\lambda}^\cH)^T & (\bm{\lambda}^\cG)^T \end{matrix} \right]^T,\\
\bM^\prime &= \left[ \begin{matrix}
\bZero & \bZero & -\alpha\bC & \bZero & \bI \\
\bZero & \bZero & \bZero & -\bI & \bZero\\
-\alpha\bC^T & \bZero & \bZero & \bC^T & \bZero\\	
\bZero & \bI & -\bC & \bZero & \bZero  \\	
-\bI & \bZero & \bZero & \bZero & \bZero  \\	
\end{matrix} \right],\\
\bq^\prime &= \left[ \begin{matrix} (\alpha\bb+\bone)^T & \alpha\bone^T & \bZero^T & \bb^T & \bone^T \end{matrix} \right]^T .
\end{aligned}
\end{align}
Finding a feasible $\bz^\prime$ for~\eqref{lcp} is an instance of a \emph{linear complementarity problem (LCP)}~\cite{mangasarian1978characterization}. Define the convex set
\begin{align}
\cP^\prime = \{ \bz^\prime  \in \mathbb{R}^{4M+d+1} \mid \bM^\prime\bz^\prime + \bq^\prime \ge \mathbf{0}, \;\; \bz^\prime \ge \mathbf{0} \}.
\end{align}
We invoke the following result from~\cite[Lemma 2]{mangasarian1978characterization}.

\begin{theorem}\label{thm:lcp}
	If the LCP defined by the constraints~\eqref{lcp} has a solution, then it has a solution at a vertex of $\cP^\prime$.
\end{theorem}

Theorem~\ref{thm:lcp} implies that the KKT points of~\eqref{penalized} (including the solutions of the problem) occur at the vertices of $\cP^\prime$. This also implies that~\eqref{penalized} has a vertex solution, viz.:

\begin{theorem}\label{thm:lcp2}
	For any vertex
	\begin{align}
	\bz^\prime_v = [\bz_v^T~~(\bm{\lambda}^\cH_v)^T~~(\bm{\lambda}^\cG_v)^T)]^T
	\end{align}
	of $\cP^\prime$, $\bz_v$ is a vertex of $\cP$.
\end{theorem}
\begin{proof}
	%Convert the linear constraints in $\cP$ and $\cP^\prime$ into standard form
	%\begin{align}
	%\begin{aligned}
	%\left[ \begin{matrix} \bM & \bI \end{matrix} \right]\left[ \begin{matrix} \bz \\ \bm{\gamma} \end{matrix} \right] + \bq = \bZero, \; \left[ \begin{matrix} \bz \\ \bm{\gamma} \end{matrix} \right] \ge \bZero, \\
	%\left[ \begin{matrix} \bM^\prime & \bI \end{matrix} \right]\left[ \begin{matrix} \bz^\prime \\ \bm{\gamma}^\prime \end{matrix} \right] + %\bq^\prime = \bZero, \; \left[ \begin{matrix} \bz^\prime \\ \bm{\gamma}^\prime \end{matrix} \right] \ge \bZero, \\
	%\end{aligned}
	%\end{align}
	%where $\bm{\gamma}$ and $\bm{\gamma}^\prime$ are additional slack variables.
	If $\bz_v^\prime$ is a vertex of $\cP^\prime$, then, there is a diagonal matrix $\bE$ such that
	\begin{align}
	\bM^\prime \bE \bz^\prime_v + \bq^\prime - \bm{\gamma}^\prime = \bZero,
	\end{align}
	where $\bE_{i,i} = 1$ if the $i$-th column of $\bM^\prime$ appears in the basic solution corresponding to vertex $\bz_v^\prime$, and $\bE_{i,i} = 0$ otherwise (the non-negative vector $\bm{\gamma}^\prime$ contains the values of additional slack variables to convert the constraints in $\cP^\prime$ into standard form). Let $\bM^\prime_J$ be the last-$2M$ rows of $\bM^\prime$. Then, 
	\begin{align}
	\bM^\prime_J \bE \bz^\prime_v + \left[ \begin{matrix} \bb^T & \bone^T \end{matrix} \right]^T - \bm{\gamma}^\prime_J = \bZero,
	\end{align}
	where $\bm{\gamma}^\prime_J$ is the last-$2M$ elements of $\bm{\gamma}^\prime$. Note that, since the right-most $2M \times 2M$ submatrix of $\bM^\prime_J$ is a zero matrix (see~\eqref{bigM}), then
	\begin{align}\label{basic}
	\bM^\prime_J \bE_K \bz_v + \left[ \begin{matrix} \bb^T & \bone^T \end{matrix} \right]^T - \bm{\gamma}^\prime_J = \bZero,
	\end{align}
	where $\bE_K$ is the first-$(2M+d+1)$ columns of $\bE$. Since $\bM^\prime_J \bE_K = \bM$, then~\eqref{basic} implies that $\bz_v$ is a vertex of $\cP$.\qed
\end{proof}

\subsubsection{Frank-Wolfe algorithm}

Theorem~\ref{thm:lcp2} suggests an approach to solve~\eqref{penalized} by searching for a vertex solution. Further, note that for a fixed $\bu$,~\eqref{penalized} reduces to an LP. Conversely, for fixed $\bs$ and $\bv$,~\eqref{penalized} is also an LP. This advocates alternating between optimizing subsets of the variables using LPs. Algorithm~\ref{alg:FW_Penalty} summarizes the method, which is in fact a special case of the Frank-Wolfe method~\cite{fw56} for non-convex quadratic minimization.

\begin{algorithm}[ht]\centering
	\caption{Frank-Wolfe method for~\eqref{penalized}.}
	\label{alg:FW_Penalty}                         
	\begin{algorithmic}[1]                   
		\REQUIRE Data $\{ \bc_i, b_i \}^{M}_{i=1}$, penalty value $\alpha$, initial solution $\bu^{(0)}$, $\bv^{(0)}$, $\bs^{(0)}$, threshold $\delta$.
		\STATE $P^{(0)} \leftarrow P(\bu^{(0)}, \bs^{(0)}, \bv^{(0)} \mid \alpha)$.
		\STATE $t \leftarrow 0$.
		\WHILE{true}
		\STATE $t \leftarrow t+1$.
		\STATE $\bs^{(t)}, \bv^{(t)} \leftarrow \argmin_{\bs,\bv} P(\bu^{(t-1)}, \bs, \bv \mid \alpha)$ s.t.~$\cP$.\label{line:lp1}
		\STATE $\bu^{(t)} \leftarrow \argmin_{\bu} P(\bu, \bs^{(t)}, \bv^{(t)} \mid \alpha)$ s.t.~$\cP$.\label{line:lp2}
		\STATE $P^{(t)} \leftarrow P(\bu^{(t)}, \bs^{(t)}, \bv^{(t)} \mid \alpha)$.
		\IF {$|P^{(t-1)} - P^{(t)}| \le \delta$}
		\STATE Break.
		\ENDIF			
		\ENDWHILE
		\RETURN $\bu^{(t)},\bv^{(t)}, \bs^{(t)}$.
	\end{algorithmic}
\end{algorithm}

\begin{theorem}\label{thm:vertex}
	In a finite number of steps, Algorithm~\ref{alg:FW_Penalty} converges to a KKT point of~\eqref{penalized}.
\end{theorem}
\begin{proof}
	The set of constraints $\cP$ can be decoupled into the two disjoint subsets
	\begin{align}
	\cP = \cP_1 \times \cP_2,
	\end{align}
	where $\cP_1$ involves only $\bs$ and $\bv$, and $\cP_2$ is the complement of $\cP_1$. With $\bu$ fixed in Line~\ref{line:lp1}, the LP converges to a vertex of $\cP_1$. Similarly, with $\bs$ and $\bv$ fixed in Line~\ref{line:lp2}, the LP converges to a vertex in $\cP_2$. Each intermediate solution $\bu^{(t)},\bv^{(t)}, \bs^{(t)}$ is thus a vertex of $\cP$ or a KKT point of~\eqref{penalized}. Since each LP must reduce or maintain $P(\bz \mid \alpha)$ which is bounded below, the process terminates in finite steps.
	\qed
\end{proof}

\noindent\textbf{Analysis of update steps} A closer look reveals the LP in Line~\ref{line:lp1} (Algorithm~\ref{alg:FW_Penalty}) to be
\begin{equation}\label{lp1}\tag{LP1}
\begin{aligned}
& \min_{\bs, \bv} && \sum_{i} s_i - u_i(\bc_i^T \bv - b_i)\\
& \text{s.t.} && s_i - \bc_i^T\bv + b_i  \ge 0,\\
& && s_i, v_i \ge 0,\\
\end{aligned}
\end{equation}
and the LP in Line~\ref{line:lp2} (Algorithm~\ref{alg:FW_Penalty}) to be
\begin{equation}\label{lp2}\tag{LP2}
\begin{aligned}
& \min_{\bu} && \sum_{i} u_i\left[1 - \alpha(\bc_i^T \bv - b_i)\right]\\
& \text{s.t.} && 0 \le u_i \le 1.
\end{aligned}
\end{equation}
Observe that LP2 can be solved in closed form and it also drives $\bu$ to integrality: if $[1 - \alpha(\bc_i^T \bv - b_i)] \le 0$, set $u_i = 1$, else, set $u_i = 0$. Further, LP1 can be seen as ``weighted" $\ell_1$-norm minimization, with $\bu$ being the weights. Intuitively, therefore, Algorithm~\ref{alg:FW_Penalty} alternates between residual minimization (LP1) and inlier-outlier dichotomization (LP2).

\subsection{Main algorithm}\label{sec:main}

%\subsubsection{Exactness of penalization}

%In this section, we investigate the conditions under which solving~\eqref{penalized} is equivalent to solving~\eqref{minout3}, before describing our main algorithm for consensus maximization.

%\subsection{Exactness of penalization}\label{sec:exactness}

%The penalty problem~\eqref{penalized} is an instance of a non-smooth exact penalty method~\cite[Sec.~17.2]{nocedal}. Observe that $Q(\bz)$ is the $\ell_1$-norm of the LHS of the equality constraints in~\eqref{minout3}. The exactness of the penalization is exhibited in the following theorems (rephrased in the context of our problem).

%\begin{theorem}[based on Theorem 17.3 in~\cite{nocedal}]
%	If $\bz^\ast$ is a local solution of the original problem~\eqref{minout3}, then, there exists $\alpha^* > 0$ such that for all $\alpha \ge \alpha^*$, $\bz^\ast$ is also a local minimizer of $P(\bz \mid \alpha)$ subject to constraints $\cP$.
%\end{theorem}

Intuitively, if the penalty parameter $\alpha$ is small, Algorithm~\ref{alg:FW_Penalty} will pay more attention to minimizing $\sum_i u_i$ and less attention to ensuring that the optimized variables are feasible w.r.t.~the original problem~\eqref{minout3}. Conversely, if $\alpha$ is large, the complementarity residual $Q(\bz)$ will be reduced more aggressively, thus the optimized $\bz$ tends to be ``more feasible". If $\alpha$ is sufficiently large, $Q(\bz)$ will be reduced to zero, and any movement to attempt to reduce $\sum_i u_i$ will not payoff, thus preserving the feasibility of $\bz$--- Section~\ref{sec:pen_converge} will formally establish this condition.

%A follow-up theorem will prove more useful for our aims.

%\begin{theorem}[based on Theorem 17.4 in~\cite{nocedal}]
%	Let $\hat{\bz}$ be a KKT point of the penalized problem~\eqref{penalized} for $\alpha$ greater than $\alpha^\ast$. Then, $Q(\hat{\bz}) = 0$, and $\hat{\bz}$ is also a KKT point of~\eqref{minout3}.
%	\label{thm:firstorder}
%\end{theorem}

%A ``one shot" approach that sets $\alpha$ to a very large value and solves a single instance of~\eqref{penalized} is unlikely to be successful, however, since we cannot globally solve the penalty problem. In the next section, we describe a more practical approach that uses an increasing sequence of $\alpha$.

%\subsubsection{Main algorithm}

The above observations argue for a deterministic consensus maximization algorithm based on solving~\eqref{penalized} for progressively larger $\alpha$'s; see Algorithm~\ref{alg:Heuristics}. For each $\alpha$, our method solves~\eqref{penalized} using Algorithm~\ref{alg:FW_Penalty}. The solution $\hat{\bz}$ for a particular $\alpha$ is then used to initialize Algorithm~\ref{alg:FW_Penalty} for the next larger $\alpha$. The sequence terminates when the complementarity residual $Q(\bz)$ vanishes or becomes insignificant.

%As long as each $\hat{\bz}$ is a KKT point of the associated penalty problem~\eqref{penalized}, which we can provably achieve thanks to Theorem~\ref{thm:vertex}, Theorem~\ref{thm:firstorder} guarantees that Algorithm~\ref{alg:Heuristics} will converge to a solution for~\eqref{minout3} that satisfies the first-order necessary conditions for optimality.

\begin{algorithm}[ht]\centering
	\caption{Non-smooth penalty method for solving~\eqref{minout3}.}
	\label{alg:Heuristics}                         
	\begin{algorithmic}[1]                   
		\REQUIRE Data $\{\bc_i, b_i\}^{M}_{i=1}$, initial model parameter $\btheta$, initial penalty value $\alpha$, increment rate $\kappa$, threshold $\delta$.
		\STATE $\bv \leftarrow \left[ (\btheta+|\min_j(\theta_j)|\bone)^T~~|\min_j(\theta_j)| \right]^T$.
		\STATE $\bu \leftarrow \mathbb{I}(\bC \bv - \bb > 0)$.
		\STATE $\bs \leftarrow \bu \odot (\bC \bv - \bb)$.
		\WHILE{true}
		\STATE $\bu, \bs, \bv \leftarrow FW(\{ \bc_i, b_i \}^{M}_{i=1}, \alpha, \bu, \bs, \bv)$.~/*Algo.~\ref{alg:FW_Penalty}.*/\label{line:fw}
		\IF {$Q(\bz) \le \delta$}
		\STATE Break.
		\ENDIF
		\STATE $\alpha \leftarrow \kappa \cdot \alpha$.
		\ENDWHILE
		\RETURN $\bu, \bs, \bv$.
	\end{algorithmic}
\end{algorithm}

It is worthwhile to note that typical non-smooth penalty functions cannot be easily minimized (e.g., no gradient information). In our case, however, we exploited the special property of~\eqref{penalized} (Sec.~\ref{sec:stationarity}) to enable efficient minimization.

\subsubsection{Convergence}\label{sec:pen_converge}

\begin{theorem}\label{thm:penalty_convergence}
If $\alpha$ is sufficiently large, Algorithm~\ref{alg:Heuristics} converges to a point $\hat{\bz}$ where $Q(\hat{\bz}) = 0$, i.e., $\hat{\bz}$ is a feasible solution of problem~\eqref{minout3}.
\end{theorem}

\begin{proof}
Let $\hat{\bs}$ and $\hat{\bv}$ be the solution of~\ref{lp1} (for a fixed $\hat{\bu}$ from the previous iteration).  When updating $\bu$ in~\ref{lp2}, for each constraint $i$, the possible outcomes for $u_i$ are:
\begin{itemize}
\item If $\bc_i^T\hat{\bv} - b_i \le 0$: We say that the $i$-th constraint is consistent with $\hat{\bv}$. \ref{lp2} will set $u_i$ to $0$ regardless of $\alpha$.
\item If $\bc_i^T\hat{\bv} - b_i > 0$: We say that the $i$-th constraint violates $\hat{\bv}$. \ref{lp2} will set $u_i$ according to
\begin{align*}
u_i = \begin{cases}
0 & \text{if} \;\; 1 - \alpha (\bc_i^T\hat{\bv} -b_i) \ge 0, \\
1 & \text{if} \;\; 1 - \alpha (\bc_i^T\hat{\bv} -b_i) < 0.
\end{cases}
\end{align*}
\end{itemize}
If $\alpha$ is large enough, then \ref{lp2} will set $u_i = 1$ for all the violating constraints. Given a $\hat{\bu}$ that was obtained under such a sufficiently large $\alpha$ in \ref{lp2}, in the subsequent invocation of \ref{lp1}, the minimal cost of $0$ can be obtained by maintaining the previous $\hat{\bv}$ and setting
\begin{align*}
\hat{s}_i = & \begin{cases}
0 & \text{if} \;\; \hat{u}_i = 0,\\
\bc_i^T\hat{\bv} - b_i & \text{if} \;\; \hat{u}_i = 1.\\
\end{cases}
\end{align*}
Recognizing that the objective function of \ref{lp1} is equal to $Q(\bz)$ completes the proof.\qed
\end{proof}

\subsubsection{Initialization}
Algorithm~\ref{alg:Heuristics} requires the initialization of $\bu$, $\bs$ and $\bv$. For consensus maximization, it is more natural to initialize the model parameters $\btheta$, which in turn gives values to $\bv$, $\bs$ and $\bu$. In our work, we initialize $\btheta$ as the least squares solution, or by executing RANSAC (Sec.~\ref{sec:results} will compare the results of these two different initialization methods).

Other required inputs are the initial penalty parameter $\alpha$ and the increment rate $\kappa$. These values affect the convergence speed of Algorithm~\ref{alg:Heuristics}. To avoid bad minima, we set $\alpha$ and $\kappa$ conservatively, e.g., $\alpha \in [1,10]$, $\kappa \in [1,~5]$. As we will demonstrate in Sec.~\ref{sec:results}, these settings enable Algorithm~\ref{alg:Heuristics} to find very good solutions at competitive runtimes.

\section{ADMM-based algorithm}\label{sec:admm}

Our second technique derives from the class of proximal splitting algorithms~\cite{boyd2011distributed}. Specifically, we apply the ADMM to construct a deterministic approximate algorithm for our target problem~\eqref{minout3}. The ADMM was originally developed for convex optimization problems. However, its use for nonconvex nonsmooth optimization has been investigated recently, with strong convergence results~\cite{hong2016convergence, wang2015global}. While ADMM has recently found usage in several geometric vision problems, e.g., bundle adjustment~\cite{eriksson2015high, eriksson2016consensus}, triangulation~\cite{eriksson2014pseudoconvex}, its application to robust fitting is relatively unexplored.

\subsection{ADMM formulation}

The specific version of ADMM used in our work is \emph{consensus ADMM}~\cite{boyd2011distributed}, where the term ``consensus" takes a different meaning\footnote{Consensus ADMM is a version commonly used for distributed optimization~\cite{boyd2011distributed}. For brevity, we do not explore distributed optimization in our work, though our algorithm is amenable to such a scheme.} than ours---to avoid confusion, we will simply call the technique ``ADMM". To the original problem~\eqref{minout3}, where the objective function has $M$ summands and the original variables are $\bz = [\bu^T \; \bs^T \; \bv^T]^T \in \mathbb{R}^{2M + d + 1}$, introduce $M$ auxilary parameter vectors $\bz^\prime_1, \dots, \bz^\prime_M$, where
\begin{align}
\bz^\prime_i = [u^\prime_i \; s^\prime_i \; (\bv^\prime_i)^T]^T \in \bbR^{d+3},
\end{align}
as well as the ``coupling" parameter vector
\begin{align}
\bz_C = [\bs_C^T \; \bv_C^T ]^T \in \bbR^{M+d}.
\end{align}
Then, rewrite~\eqref{minout3} as
\begin{subequations}
  \label{consensus_admm}
  \begin{align}
  &\min_{\bz,\{\bz^\prime_i\},\bz_C} &&\sum_i \left[ u^\prime_i + \bbI_B(\bz^\prime_i) \right] + \bbI_C(\bz_C) \label{consensus_admm_cost} \\
  &\text{s.t.}     && \bu = \bu^\prime,  \label{consensus_admm_constraint1}\\
  & && \bs = \bs^\prime = \bs_C, \label{consensus_admm_constraint2}\\
  & && \bv = \bv^\prime_i = \bv_C, \label{consensus_admm_constraint3}
  \end{align}
\end{subequations}
where $\mathbb{I}_\bB$ is an indicator function that enforces the bilinear constraints
\begin{equation}
  \label{indicator_ib}
  \mathbb{I}_B(\bz^\prime_i) =
  \begin{cases}
    \begin{aligned}
      &0 &&\text{ if }
      \begin{cases}
        u^\prime_i(s^\prime_i -\bc_i^T\bv_i^\prime + b_i) = 0,
        \\ s^\prime_i(1-u^\prime_i) = 0, \\
        u^\prime_i \in \{0,1\},
      \end{cases}\\
      &\infty &&\text {otherwise,}
    \end{aligned}
  \end{cases}
\end{equation}
and $\bbI_C$ is an indicator function that enforces $\bz_C$ to statisfy the convex constraints
\begin{equation}
  \label{indicator_ic}
  \mathbb{I}_C(\bz_C) =
  \begin{cases}
    \begin{aligned}
    &0 &&\text{ if }
    \begin{cases}
      \bs_C - \bC\bv_C + \bb \ge \bZero,\\
      \bs_C, \bv_C \ge \bZero,
    \end{cases}\\
    &\infty &&\text { otherwise.}
    \end{aligned}
  \end{cases}
\end{equation}

Note that the objective function~\eqref{consensus_admm_cost} is a composition of $M+1$ totally separate subfunctions, where each subfunction of the form $u^\prime_i + \bbI_B(\bz^\prime_i)$ involves only $\bz^\prime_i$, and the final subfunction $\mathbb{I}_C(\bz_C)$ involves only $\bz_C$. Intuitively, the constraints~\eqref{consensus_admm_constraint1},~\eqref{consensus_admm_constraint2}, and~\eqref{consensus_admm_constraint3} ensure that the auxiliary and the original variables must converge to the same point, and hence are referred to as ``coupling constraints''. It can thus be appreciated that problem~\eqref{consensus_admm} is identical to problem~\eqref{minout3}, in that solving \eqref{consensus_admm} results in the same optimum as \eqref{minout3}. The benefit of the decomposition is that the problem can be solved by iteratively solving smaller subproblems which are convex, as we elaborate in the next subsection.

%Intuitively, by introducing the auxilary variables together with the indicator functions~\eqref{indicator_ib} and~\eqref{indicator_ic}, the cost function~\eqref{consensus_admm_cost} can be considered as a composition of $N+1$ totally separate functions, where each of the first $N$ functions only involve $u^\prime_i, s^\prime_i, \bv^\prime_i$ and the $(N+1)$-th function only involve $\bs_C$ and $\bv_C$. The constraints~\eqref{consensus_admm_constraint1}~\eqref{consensus_admm_constraint2},~\eqref{consensus_admm_constraint3} ensure that the auxilary and the original variables must converge to the same point in the final solution, and are referred to as ``coupling constraints''. Clearly, the formulation of \eqref{consensus_admm} is indentical to \eqref{minout3}, thus solving \eqref{consensus_admm} results in the same optimum as \eqref{minout3}. Later in this section, the advantage of function decoupling will be clarified, i.e., the solving of~\eqref{minout3} amounts to iteratively solving smaller sub-problems of convex quadratic programs which can be performed efficiently by common solvers. Henceforth, we develop the consensus ADMM iterations to solve \eqref{consensus_admm}.
It can further be realized that the solution of the problem~\eqref{consensus_admm} does not change if the term $\|\bu\|^2$ is added to the cost function~\eqref{consensus_admm_cost}. Thus, to aid the convergence of our proposed algorithm (refer to the supplementary material (Section 1) for more details), the solution of \eqref{consensus_admm} can be obtained by solving the following problem:
\begin{subequations}
	\label{consensus_admm1}
	\begin{align}
	&\min_{\bz,\{\bz^\prime_i\},\bz_C} &&\sum_i \left[ u^\prime_i + \bbI_B(\bz^\prime_i)  \right] + \bbI_C(\bz_C) + \|\bu\|^2 \label{consensus_admm1__reg_cost} \\
	&\text{s.t.}     && \bu = \bu^\prime,  \label{consensus_admm1__reg_constraint1}\\
	& && \bs = \bs^\prime = \bs_C, \label{consensus_admm1__reg_constraint2}\\
	& && \bv = \bv^\prime_i = \bv_C, \label{consensus_admm1_reg_constraint3}
	\end{align}
\end{subequations}

\subsubsection{Augmented Lagrangian}

Now consider the augmented Lagrangian of \eqref{consensus_admm1}
\begin{equation}
  \label{lagrangian}
  \begin{aligned}
  & \cL_\rho (\bz, \{\bz^\prime_i\}, \bz_C, \blambda) = && \sum_i \left[ u^\prime_i  + \bbI_B(\bz^\prime_i) \right] + \bbI_C(\bz_C) + \|\bu\|^2\\
  & &&+ \rho(\|\bu^\prime - \bu + \blambda^{\bu}\|^2_2 - \|\blambda^\bu\|^2_2)\\ 
  & &&+ \rho(\|\bs^\prime - \bs + \blambda^{\bs}\|^2_2 - \|\blambda^{\bs}\|^2_2) \\
  & &&+ \rho(\|\bs_C - \bs + \blambda_C^{\bs}\|^2_2 - \|\blambda_C^{\bs}\|^2_2)\\
  & &&+ \rho(\|\bv_C - \bv + \blambda_C^{\bv}\|^2_2 - \|\blambda_C^{\bv}\|^2_2) \\
  & &&+ \rho\sum_i (\|\bv_i^\prime - \bv + \blambda_i^{\bv}\|^2_2 - \|\blambda_i^{\bv}\|^2_2),
\end{aligned}
\end{equation}
where
\begin{align}
\bu^\prime = \left[ \begin{matrix} u^\prime_1 & \dots u^\prime_M \end{matrix} \right]^T, \;\; \bs^\prime = \left[ \begin{matrix} s^\prime_1 & \dots s^\prime_M \end{matrix} \right]^T,
\end{align}
and $\rho$ is the penalty parameter. The vector
\begin{align}
\blambda = [(\blambda^{\bu})^T \;(\blambda^{\bs})^T \; (\blambda_C^{\bs})^T \; (\blambda_C^{\bv})^T  \; \{(\blambda_i^{\bv})^T\}_{i=1}^M]^T
\end{align}
contains all the scaled dual variables associated with the constraints in~\eqref{consensus_admm1}. Intuitively, the penalty parameter $\rho$ controls the strength of the penalization of the deviation of the auxilary variables from the original ones.

ADMM alternates between updating the auxilary variables $\{\bz^\prime_i\}$ and $\bz_C$, followed by the original variables $\bz$, w.r.t.~the augmented Lagrangian. The Lagrange multipliers $\blambda$ are also updated, following the dual variable update principle~\cite{boyd2011distributed}. Sec.~\ref{sec:admm_main} will elaborate on the overall algorithm and the associated convergence guarantee. Next in Sec.~\ref{sec:admm_update} we will first examine in detail the individual update steps.

\subsection{Update steps}\label{sec:admm_update}

%The concept behind consensus ADMM is to alternatively minimize the Lagrangian with respect to every single variable while the remaining variables being fixed. Given the augmented Lagrangian \eqref{lagrangian} of the problem \eqref{consensus_admm} the ADMM update steps can be written as follows (the superscript $(t)$ denotes the iteration number)

The vectors $\{\bz^\prime_i\}$, $\bz_C$, and $\bz$ are respectively updated by minimizing the augmented Lagrangian with respect to the target vector, while keeping the other vectors fixed. Specifically, these updates are
\begin{subequations}
  \label{admm_update}
  \begin{align}
    &\bz^\prime_i \leftarrow \argmin_{\bz^\prime_i} \; \cL_\rho (\bz, \{\bz^\prime_i\}, \bz_C, \blambda), \; \forall i,\label{zi_update_1}\\
    &\bz_C\leftarrow\argmin_{\bz_C} \; \cL_\rho (\bz, \{\bz^\prime_i\}, \bz_C, \blambda),\label{zc_update_1}\\
    &\bz \leftarrow\argmin_{\bz} \; \cL_\rho (\bz, \{\bz^\prime_i\}, \bz_C, \blambda),\label{z_update_1}
  \end{align}
\end{subequations}
where, to avoid clutter, we don't distinguish between the target vector and the other vectors on the RHS.
%we assume that except for the target vector, the other vectors on the RHS are known.

After the vectors $\{\bz^\prime_i\}$, $\bz_C$, and $\bz$ are revised, the ADMM procedure updates the Lagrange multipliers as follows    
\begin{align}
  \label{admm_update2}
  \begin{aligned}    
    &\blambda^{\bu} \leftarrow \blambda^{\bu} + \bu^{\prime} - \bu,\\
    &\blambda^{\bs} \leftarrow \blambda^{\bs} + \bs^{\prime} - \bs,\\
    &\blambda_C^{\bs} \leftarrow \blambda_C^{\bs} + \bs_C - \bs,\\
    &\blambda_C^{\bv} \leftarrow \blambda_C^{\bv} + \bv_C - \bv,\\
    &\blambda_i^{\bv} \leftarrow \blambda_i^{\bv} + \bv^\prime_i - \bv, \; \forall i.
  \end{aligned}
\end{align}
Intuitively, from the way vector $\blambda$ is being updated, the vector can be interpreted as the accumulated shift of the auxiliary variables from the original variables~\cite{boyd2011distributed}.

In the following, we take a deeper look into the subproblems in~\eqref{admm_update}.

\subsubsection{Updating $\bz^\prime_i$}

Due to the decomposable nature of the augmented Lagrangian~\eqref{lagrangian}, the problem in \eqref{zi_update_1} can be reduced to
%\begin{align}
%  \begin{aligned}
%    \argmin_{\bz^\prime_i} \;\; &u^\prime_i + \bbI_B(\bz^\prime_i) + \rho(u^\prime_i - u_i + \lambda^\bu_i)^2\\
%    &+\rho(s^\prime_i - s_i + \lambda^\bs_i)^2 + \rho\|\bv_i^\prime - \bv + \blambda^{\bv}_i\|^2_2,
%  \end{aligned}
%\end{align}
\begin{subequations}
  \label{zi_update}
  \begin{align}
    &\argmin_{\bz^\prime_i}&& u^\prime_i+ \rho(u^\prime_i - u_i + \lambda^\bu_i)^2\notag\\
    & &&+\rho(s^\prime_i - s_i + \lambda^\bs_i)^2+\rho\|\bv_i^\prime - \bv + \blambda^\bv_i \|^2_2\label{ziupdate_objective}\\
    &\text{s.t.}  && u^\prime_i(s^\prime_i - \bc_i^T\bv^\prime_i + b_i) = 0, \label{bilinear_constraint_2} \\    
    & && s'_i(1 - u^\prime_i) = 0, \label{bilinear_constraint_1} \\
    & && u^\prime_i \in \{0,1\}. \label{u_constraint}
  \end{align}
\end{subequations}
where terms not affected by $\bz_i^\prime$ have also been ignored.     
Due to the complementarity constraints \eqref{bilinear_constraint_2} and \eqref{bilinear_constraint_1}, and the binary restriction \eqref{u_constraint} on $u^\prime_i$ , \eqref{zi_update} can be solved by simply enumerating $u^\prime_i$:
\begin{itemize}
\item $\uiprime = 0$: Then $\siprime$ must also be $0$ to satisfy all the constraints, and $\bv_i^\prime$ must be assigned the value of $\bv - \blambda_i^{\bv}$ to minimize \eqref{ziupdate_objective}.
\item $\uiprime = 1$: To satisfy \eqref{bilinear_constraint_2}, $\siprime$ must be equal to $\bc_i^T\bv_i^\prime - b_i$ Then problem \eqref{zi_update} becomes the unconstrained convex quadratic program (QP)
\begin{equation}
	\label{viprime_update}
	\begin{aligned}
	\min_{\bv^\prime_i} \; (\bc_i^T\bv^\prime_i - b_i - s_i + \lambda^\bs_i)^2	+\|\bv_i^\prime - \bv + \blambda^{\bv}_i\|^2_2.
	\end{aligned}
\end{equation}
When $\bv_i^\prime$ is obtained, $s^\prime_i$ can be computed accordingly.
\end{itemize}
The revised $\bz_i^\prime$ is simply chosen as the combination of the variables that results in the smaller objective value in~\eqref{zi_update}. Note that the value of $\rho$ would affect the chosen $\bz_i^\prime$.

\subsubsection{Updating $\bz_C$}

Ignoring terms unrelated to $\bz_C$, the problem in \eqref{zc_update_1} can be re-expressed as a convex QP
%\begin{equation}
%  \begin{aligned}
%\argmin_{\bz_C} \;\; &\bbI_C(\bz_C)+\rho\|\bs_C - \bs + \blambda_C^{\bs}\|^2_2 +\rho\|\bv_C - \bv + \blambda_C^{\bv}\|^2_2.
%  \end{aligned}
%\end{equation}
%The above is equivalent to the convex QP
\begin{equation}
  \label{zc_update}
  \begin{aligned}
    & \min_{\bz_C} && \|\bs_C - \bs + \blambda_C^{\bs}\|^2_2 + \|\bv_C - \bv + \blambda_C^{\bv}\|^2_2,\\  
    & \text{s.t.}  && \bs_C - \bC\bv_C + \bb \ge \bZero,\\
    & && \bs_C, \bv_C \ge \bZero,
  \end{aligned}
\end{equation}
which can be solved efficiently up to global optimality.
%where the constraints in $\bbI_C$ are now listed explicitly.

\subsubsection{Updating $\bz$}

Again ignoring terms unrelated to the variables of interest, the problem in \eqref{z_update_1} reduces to
\begin{equation}
  \label{z_update}
  \begin{aligned}
    &\argmin_{\bz} && \rho(\|\bu^\prime - \bu + \blambda^{\bu}\|^2_2 + \|\bs^\prime - \bs + \blambda^{\bs}\|^2_2 \\
    & && + \|\bs_C - \bs + \blambda_C^{\bs}\|^2_2  + \|\bv_C - \bv + \blambda_C^{\bv}\|^2_2 \\
    & && + \sum_i\|\bv_i^\prime - \bv + \blambda_i^{\bv}\|^2_2)+ \|\bu\|^2.
  \end{aligned}
\end{equation}
The three components $\bu$, $\bs$ and $\bv$ of $\bz$ decouple, and in fact can be solved for easily as the ``mean vectors"
\begin{equation*}
  \begin{aligned}
  \bu  &= \frac{\rho}{\rho+1}(\bu^\prime + \blambda^{\bu}),\\  
  \bs  &= \frac{1}{2}\left(\bs^\prime + \blambda^{\bs} + \bs_C + \blambda_C^{\bs}\right),\\
  \bv  &= \frac{1}{M+1}\left[\sum_{i=1}^M(\bv^\prime_i + \blambda_i^\bv) + \bv_C + \blambda^{\bv}_C \right].
  \end{aligned}  
\end{equation*}

Finally, we emphasize that all the update steps above can be solved efficiently, requiring no more than a convex QP.

\subsection{Main algorithm}\label{sec:admm_main}

Similar to the non-smooth penalty algorithm discussed in Sec.~\ref{sec:main}, directly setting $\rho$ to a very large value will likely lead to a bad suboptimal result. Therefore, also applied here is a heuristic strategy that initializes $\rho$ to a small value then gradually increases $\rho$ after each ADMM update cycle. The algorithm is terminated when the variable $\bz$ converges. Algorithm~\ref{alg:admm} summarizes the overall procedure.

\begin{algorithm}[ht]\centering
	\caption{ADMM-based method for solving~\eqref{minout3}.}
	\label{alg:admm}                         
	\begin{algorithmic}[1]                   
          \REQUIRE Data $\{\bc_i, b_i\}^{M}_{i=1}$, initial model parameter $\btheta$, initial penalty value $\rho$, increment rate $\sigma$, threshold $\delta$.
                \STATE $t \leftarrow 0$
		\STATE $\itt{\bv} \leftarrow \left[ (\btheta+|\min_j(\theta_j)|\bone)^T~~|\min_j(\theta_j)| \right]^T$.
		\STATE $\itt{\bu} \leftarrow \mathbb{I}(\bC \bv - \bb > 0)$.
		\STATE $\itt{\bs} \leftarrow \bu \odot (\bC \bv - \bb)$.
	    \STATE $\itt{\bz_i} = \itt{\bz}\;; \itt{\bz_C} = [\itt{\bs}; \itt{\bv}]\;; \itt{\blambda} = \bZero $
		 \WHILE{true}
		\STATE $t \leftarrow t+1$
		\STATE Update $\itt{\bz_i}$ by solving~\eqref{zi_update} $\forall\; i=1..N$
		\STATE Update $\itt{\bz_C}$ by solving~\eqref{zc_update}
		\STATE Update $\itt{\bz}$ by solving~\eqref{z_update}                
		\IF {$\|\itt{\bz} - \bz^{(t-1)} \|\le \delta$}
		\STATE Break.
		\ENDIF
		\STATE $\itt{\rho} \leftarrow \sigma \cdot \rho^{(t-1)}$.
		\ENDWHILE
		\RETURN $\bu, \bs, \bv$.
	\end{algorithmic}
\end{algorithm}

\subsubsection{Convergence}

\begin{theorem}\label{theorem:admm_converge}
  For a sufficiently large $\rho$, the ADMM update iterations in \eqref{admm_update} converge to  a stationary point of the augmented Lagrangian ~\eqref{lagrangian}, which is also a feasible solution of~\eqref{minout3}, after a finite number of steps.
\end{theorem}

\begin{proof}
  The detailed proof for this theorem can be found in the supplementary material (Section 1). For completeness, an outline of the proof is provided in this section. 
  
  Consider the $(t+1)$-th update cycle of Algorithm.~\ref{alg:admm}. To prevent clutter, let $\{\bz_i\}^+, \bz_C^+, \bz^+$ and $\blambda^+$ denote the updated value of the variables while $\{\bz_i\}, \bz_C, \bz$ and $\blambda$ represent the variables carried from the $(t)$-th iteration.
  
  During the update steps of $\{\bz_i\}$ and $\bz_C$, since ~\eqref{zi_update} and~\eqref{zc_update} can be solved optimally, it follows that:
  \begin{equation} \label{eq:lagrange_dec_1}
	 \cL_\rho (\bz, \{\bz^\prime_i\}, \bz_C, \blambda) \ge \cL_\rho (\bz, \{\bz^\prime_i\}^+, \bz_C^+, \blambda) 
  \end{equation}

  Then, after $\bz$ and $\blambda$ are updated, with a sufficiently large $\rho$, it can be proven that: 
  \begin{equation} \label{eq:larange_dec_2}
  \cL_\rho (\bz, \{\bz^\prime_i\}^+, \bz_C^+, \blambda) \ge \cL_\rho (\bz^+, \{\bz^\prime_i\}^+, \bz_C^+, \blambda^+) 
  \end{equation}	
  (detailed proof is provided in the supplementary material -- Section 1).
  From~\eqref{eq:lagrange_dec_1} and~\eqref{eq:larange_dec_2}, the following inequality holds:
  \begin{equation} \label{eq:larange_dec}
  \cL_\rho (\bz, \{\bz^\prime_i\}, \bz_C, \blambda) \ge \cL_\rho (\bz^+, \{\bz^\prime_i\}^+, \bz_C^+, \blambda^+) 
  \end{equation}
  given that $\rho$ is large enough.
     
  The inequality \eqref{eq:larange_dec} states that, with a sufficiently large $\rho$, the augmented Lagrangian~\eqref{lagrangian} is monotonically non-increasing after every ADMM update cycle. As this function is bounded below with a sufficiently large $\rho$ (detailed proof is given in the supplementary material--Section 1), its convergence to a point $\bz^*$ is guaranteed. At convergence, all the constraints~\eqref{consensus_admm_constraint1}, ~\eqref{consensus_admm_constraint2} and ~\eqref{consensus_admm_constraint3} are satisfied and $\bz^*$ is also a feasible solution of~\eqref{minout3}.
  \qed 
  
\end{proof}

\subsubsection{Initialization}
Similar to Alg.~\ref{alg:Heuristics}, $\bu,\bs,\bv$ can be initialized from a suboptimal solution such as RANSAC or least squares fit. To avoid bad local minmima, the starting values of $\rho$ are chosen to be relatively small ($0\le \rho\le 10$) with a conservative increase rate $\sigma$ ($1.01\le \sigma\le 5$). It will be demonstrated in Section~\ref{sec:results} that with this choice of the parameters, the algorithm was able to significantly improve the solution from an initial starting point.

\section{Handling geometric distances}\label{sec:quasiconvex}

For most applications in computer vision, the residual function used for geometric model fitting is non-linear. It has been shown~\cite{kahl05,olsson07,agarwal08}, however, that many geometric residuals have the following \emph{generalized fractional} form
\begin{equation}\label{quasi-convex}
\frac{\|\bG\btheta + \bh\|_p}{\br^T\btheta + q} \;\; \text{with} \;\; \br^T\btheta + q > 0,
\end{equation}
where $\| \cdot \|_p$ is the $p$-norm, and $\bG \in \mathbb{R}^{2 \times d}$, $\bh \in \mathbb{R}^{2}$, $\br \in \mathbb{R}^{d}$, $q \in \mathbb{R}^1$ are constants derived from the input data. For example, the reprojection error in triangulation and transfer error in homography fitting can be coded in the form~\eqref{quasi-convex}. The associated maximum consensus problem is
\begin{align}\label{maxcon-geo}
\begin{aligned}
&\max_{\btheta \in \mathbb{R}^d}  && \Psi(\btheta),
\end{aligned}
\end{align}
where
\begin{align}\label{geo-consensus}
\Psi(\btheta) = \sum_{j=1}^N \mathbb{I}\left(\|\bG_j\btheta + h_j\|_p \le \epsilon (\br_j^T\btheta + q_j) \right).
\end{align}
%\begin{align}\label{maxcon-geo}
%\begin{aligned}
%&\max_{\btheta \in \mathbb{R}^d,~\cI \in \cP(N)}  && |\cI| \\
%&\text{subject to} && \|\bG_j\btheta + h_j\|_p \le \epsilon (\br_j^T\btheta + q_j)  && \forall j \in \cI,\\
%&&& \br_j^T\btheta + q_j > 0 &&& \forall j \in \cI.
%\end{aligned}
%\end{align}
In~\eqref{geo-consensus}, we have moved the denominator of~\eqref{quasi-convex} to the RHS since $\epsilon$ is non-negative (see~\cite{kahl05} for details). We show that for $p = 1$, our method can be easily adapted to solve maximum consensus for geometric residuals~\eqref{maxcon-geo}\footnote{Note that, in the presence of outliers, the residuals are no longer i.i.d. Normal. Thus, the $1$-norm is arguably as valid as the $2$-norm for maximum consensus robust fitting.}. Define
\begin{equation}
\bG_j = \stackedVector{\bg_{j,1}^T}{\bg_{j,2}^T} \;\; \bh_j = \stackedVector{h_{j,1}}{h_{j,2}}.
\end{equation}
Now, for  $p=1$, the constraint in~\eqref{maxcon-geo} becomes
\begin{align}
\left|\bg^T_{j,1}\btheta + h_{j,1} \right| + \left|\bg^T_{j,2}\btheta + h_{j,2} \right| \le \epsilon (\br_j^T\btheta + q_j),
\end{align}
which in turn can be equivalently implemented using four linear constraints (see~\cite{agarwal08} for details). We can then manipulate~\eqref{geo-consensus} into the form~\eqref{consensus_maxfs}, and the rest of our theory and algorithms will be immediately applicable.

\section{Results}\label{sec:results}

\begin{figure*}[ht]\centering
	\subfigure{\includegraphics[width=1\columnwidth]{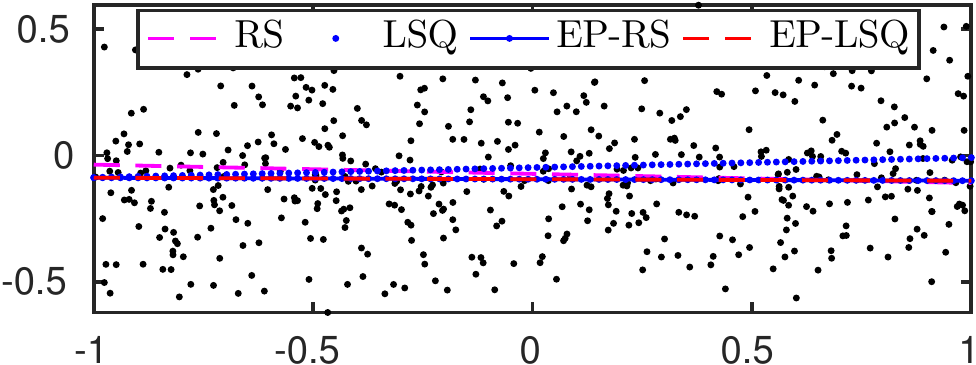}\label{subfig:balanced2d}}	
	\subfigure{\includegraphics[width=1\columnwidth]{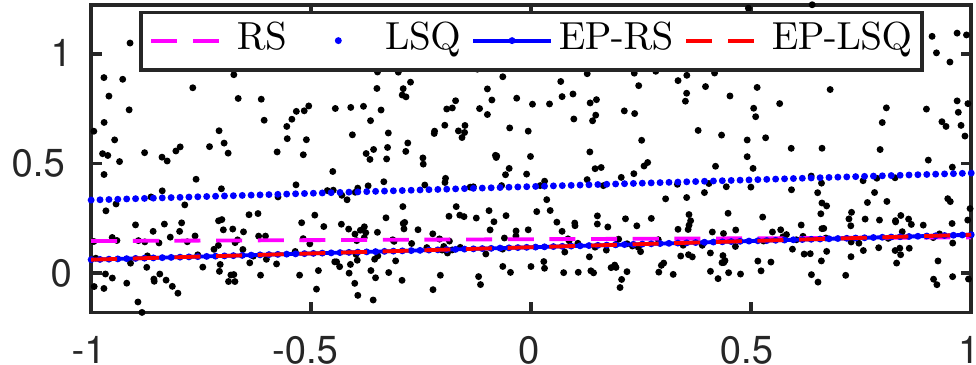}\label{subfig:unbalanced2d}}					
	\caption{Two-dimensional analogy of balanced (left) and unbalanced (right) data generated in our experiments. The results of RANSAC, least squares, and our method initialized with the former two methods are shown. Observe that least squares is heavily biased under unbalanced data, but EP is able to recover from the bad initialization. (For clarity, the results of AM variants are not plotted as they are very close to EP-RS and EP-LSQ) }
	\label{fig:data2d}
\end{figure*}
\begin{figure*}[ht]\centering
	\subfigure[Consensus size at termination (balanced data).]{\includegraphics[width=0.99\columnwidth]{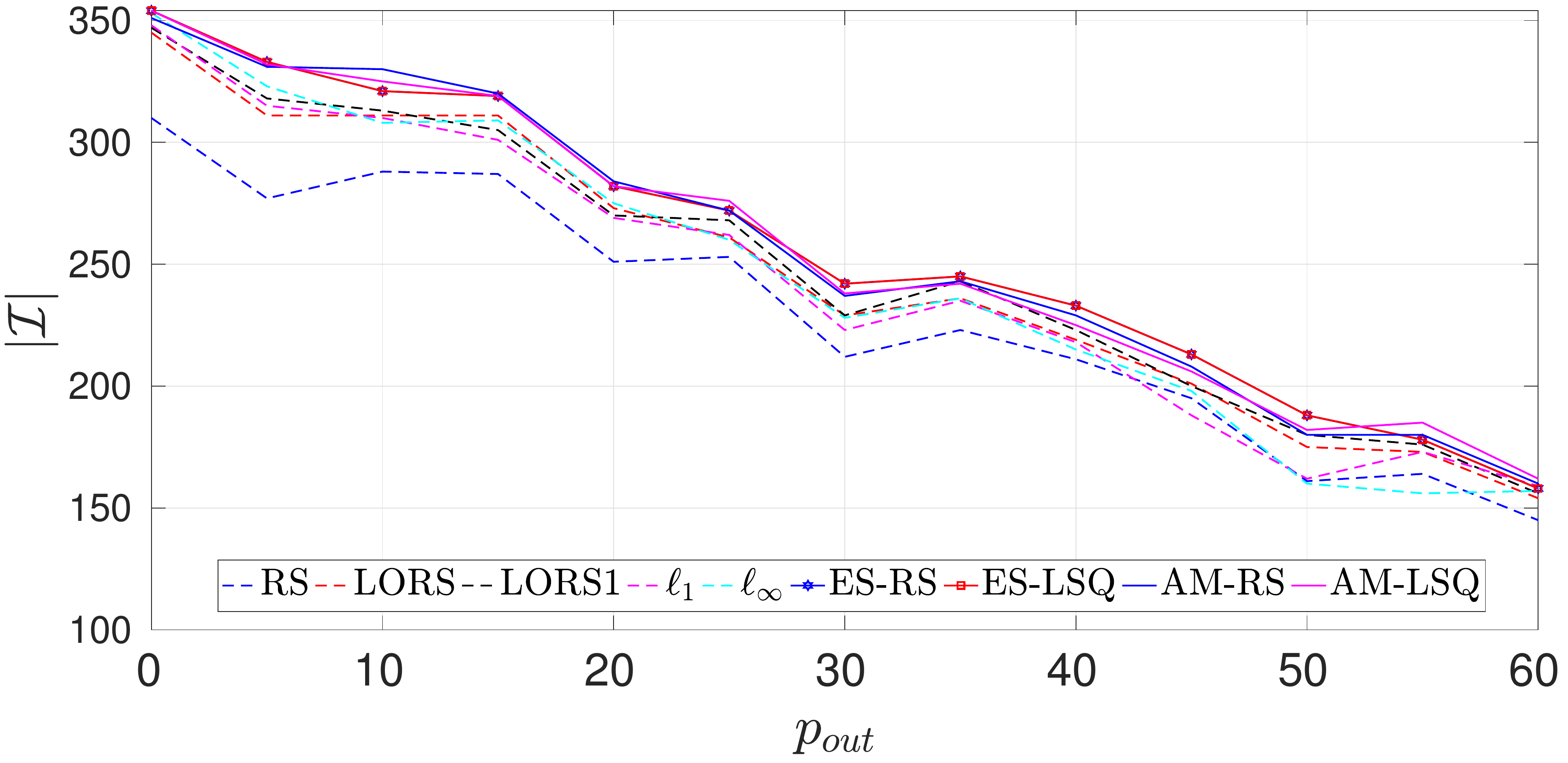}\label{subfig:linear_balanced_inls}}
	\hspace{1em}
	\subfigure[Runtime in seconds (log scale, balanced data).]{\includegraphics[width=0.99\columnwidth]{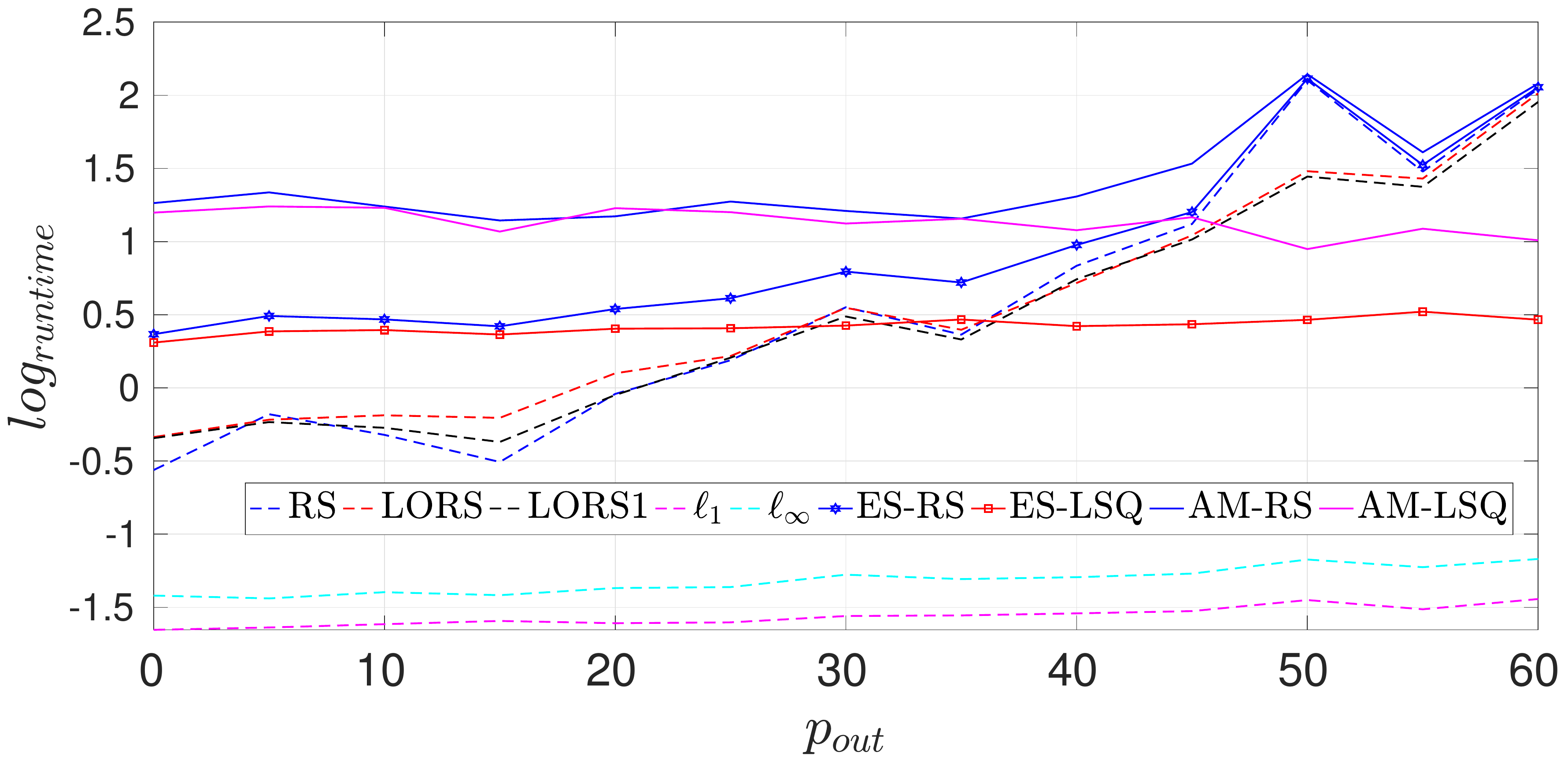}\label{subfig:linear_balanced_time}}\\
	\subfigure[Consensus size at termination (unbalanced data).]{\includegraphics[width=0.99\columnwidth]{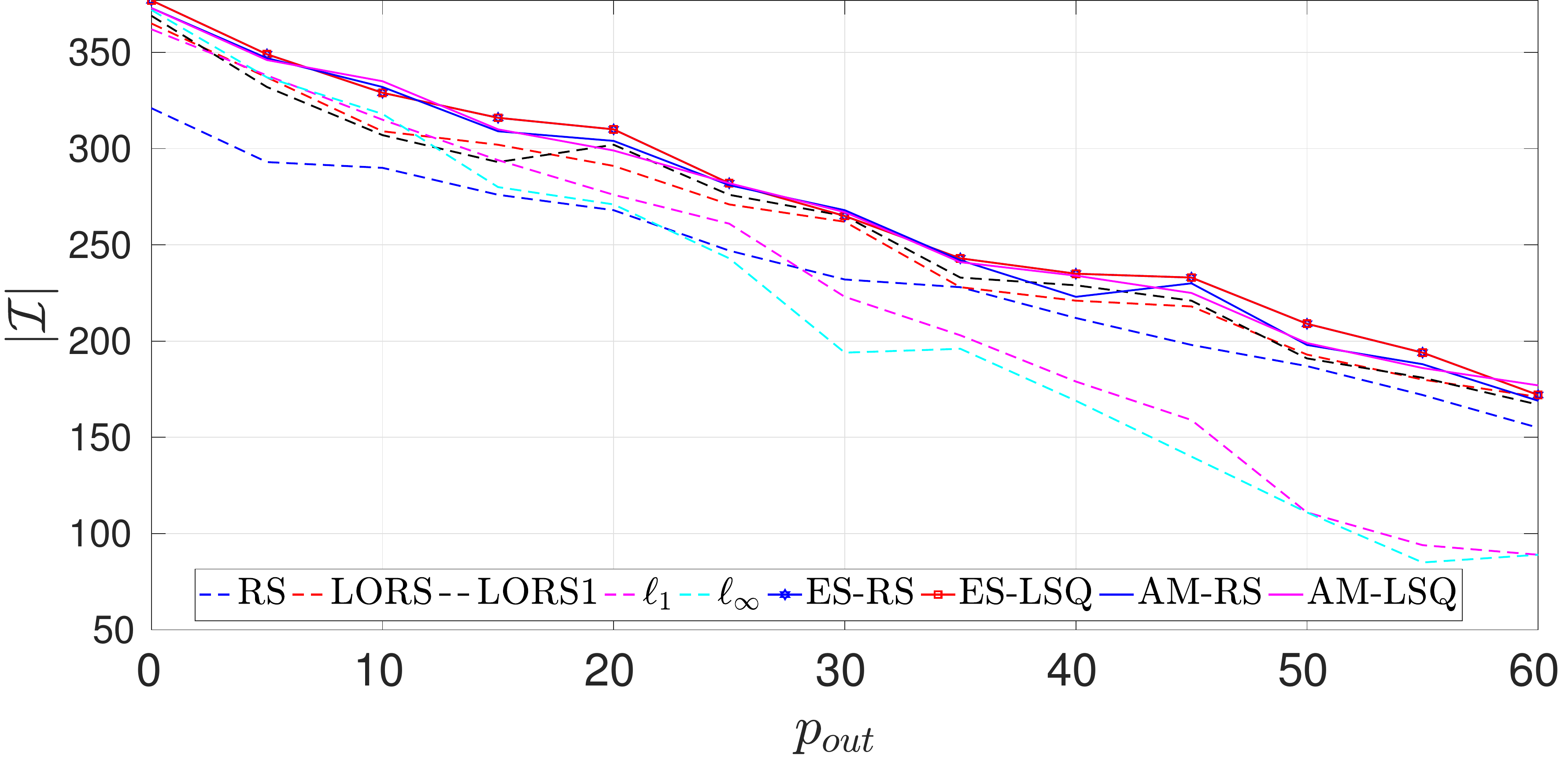}\label{subfig:linear_unbalanced_inls}}
	\hspace{1em}
	\subfigure[Runtime in seconds (log scale, unbalanced data).]{\includegraphics[width=0.99\columnwidth]{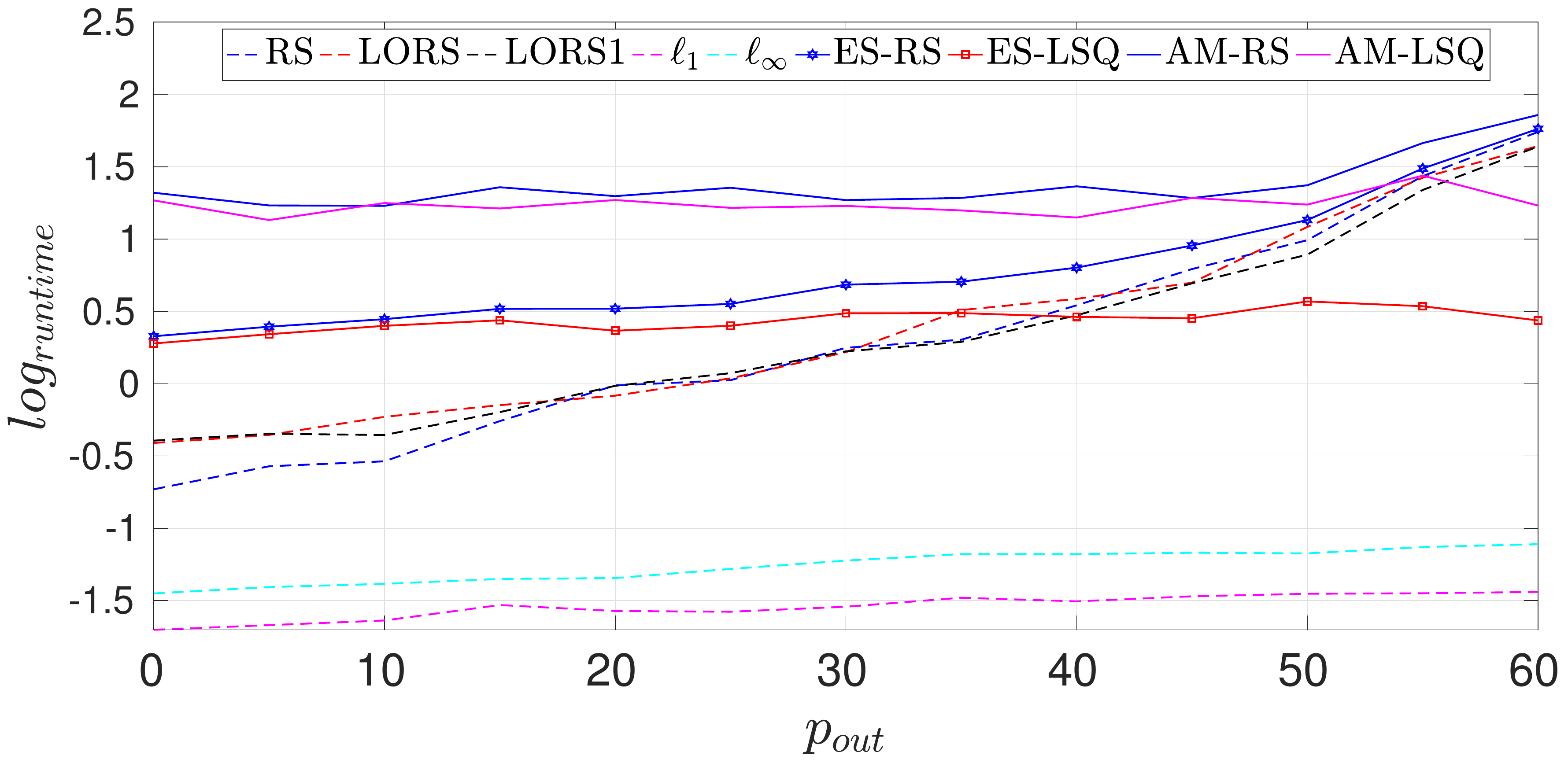}\label{subfig:linear_unbalanced_time}}
	\caption{Results for linear regression ($d = 8$ dimensions). (a)(b) Balanced data; (c)(d) Unbalanced data.}
	\label{fig:synthetic_results}
\end{figure*}
We tested our method (Algorithm~\ref{alg:Heuristics} and Algorithm~\ref{alg:admm}, henceforth abbreviated as \ep~and \admm, respectively) on common parameter estimation problems. We compared \ep~and \admm~against the following well-known methods:
\begin{itemize}[leftmargin=0.6cm,itemsep=2pt,parsep=0pt]
	\item RANSAC (RS) \cite{fischler1981random}: We used confidence $\rho = 0.99$ for the stopping criterion in all the experiments. On each data instance, RANSAC was executed 10 times and the average consensus size and runtime were reported. 
	\item LO-RANSAC (LORS) \cite{chum2003locally}: The maximum number of iterations for the inner sampling over the best consensus set was set to 100. The size of the inner sampled subsets was set to be twice the size of the minimal subset. 	
	\item Improved LO-RANSAC (LORS1) \cite{lebeda2012fixing}: Following~\cite{lebeda2012fixing}, the inner RANSAC routine will only be run if the new consensus size is higher than a pre-defined threshold (set to $10\%$ of the data size in our experiments).
	\item $\ell_1$ approximation ($\ell_1$) \cite{olsson2010outlier}: This method is equivalent to introducing slack variables to problem~\eqref{consensus} and minimizing the $\ell_1$-norm of the slack variables to yield an approximate solution to maximum consensus.
	\item $\ell_\infty$ outlier removal ($l_\infty$) \cite{sim2006removing}: Again, in the context of~\eqref{consensus}, slack variables are introduced and the maximum slack value is minimized. Data with the largest slack value are removed, and the process of repeated until the largest slack value is not greater than zero.
	\item For fundamental matrix estimation and linearized homography, we also compare our methods with Cov-RANSAC (CRS) \cite{raguram2009exploiting}, in which the uncertainties of the measurements and the homography matrix are incorporated to improve RANSAC.
	\item For the experiments with image data where key-point matching scores are available as inlier priors, we executed two state-of-the-art RANSAC variants: PROSAC (PS) ~\cite{chum2005matching} and Guided MLESAC (GMLE)~\cite{tordoff2005guided}.
\end{itemize}
All the methods and experiments were implemented in MATLAB and run on a standard desktop machine with 3.5~GHz processor and 8~GB of RAM. For \ep, \admm, $\ell_1$ and $\ell_\infty$, Gurobi was employed as the LP and QP solver.

\subsection{Linear models}

\subsubsection{Linear regression with synthetic data}

We generated $N = 500$ points $\{\bx_j, y_j\}^{N}_{j=1}$ in $\mathbb{R}^9$ following a linear trend $y = \bx^T \btheta$, where $\btheta \in \mathbb{R}^8$ and $\bx_j \in [-1, 1]^8$ were randomly sampled. Each $y_j$ was perturbed by Gaussian noise with standard deviation of $\sigma_{in} = 0.1$. To simulate outliers, $p_{out}\%$ of $y_j$'s were randomly selected and corrupted. To test the ability of our methods to deal with bad initializations, two different outlier settings were considered:
\begin{itemize}[leftmargin=0.6cm,itemsep=2pt,parsep=0pt]
	\item Balanced data: the $y_j$ of outliers were added with Gaussian noise of $\sigma_{out}=1$. This evenly distributed the outliers on both sides of the hyperplane. 
	\item Unbalanced data: as above, but the sign of the additive noise was forced to be positive. Thus, outliers were distributed only on one side of the hyperplane. On such data, the least squares solution is heavily biased.
\end{itemize}
See Fig.~\ref{fig:data2d} for a \emph{2D analogy} of these outlier settings. We tested with $p_{out}=\{0, 5, 10 \dots, 60\}$.  The inlier threshold for maximum consensus was set to $\epsilon = 0.1$.

Our algorithms \ep~and \admm~were initialized respectively with RANSAC (variants \ep-RS and \admm-RS) and least squares (variants \ep-LSQ and \admm-LSQ). For \ep~variants, the initial $\alpha$ was set to $0.5$ and $\kappa = 5$, while initial $\rho$ of \admm~variants was set to $0.1$ and $\sigma=2.5$ for all the runs.

\input{fundamental_results.tex}
\input{linear_homography_results.tex}

\input{homography_results.tex}
\input{affine_results.tex}
\input{triangulation_results.tex}

Fig.~\ref{fig:synthetic_results} shows the average consensus size at termination and runtime (in log scale) of the methods. Note that runtime of RS and LSQ were included in the runtime of \ep-RS, \admm-RS, \ep-LSQ and \admm-LSQ, respectively. It is clear that, in terms of solution quality, the variants of \ep~and \admm~consistently outperformed the other methods. The fact that \ep-LSQ could match the quality of \ep-RS on unbalanced data attest to the ability of \ep~to recover from bad initializations. In terms of runtime, while both \ep~variants were slightly more expensive than the RANSAC variants, as $p_{out}$ increased over 35\%, \ep-LSQ began to outperform the RANSAC variants (since \ep-RS was initialized using RANSAC, its runtime also increased with $p_{out}$). \admm~variants were also able to obtain roughly the same quality as~\ep-based methods, albeit with longer runtime. This is explainable as~\admm~requires solving quaratic subproblems while only LPs are required for~\ep~variants. 

\subsubsection{Fundamental matrix estimation (with algebraic error)}

Following~\cite[Chapter 11]{hartley2003multiple}, the epipolar constraint is linearized to enable the fundamental matrix to be estimated linearly (note that the usual geometric distances for fundamental matrix estimation do not have the generalized fractional form~\eqref{quasi-convex}, thus linearization is essential to enable our method. Sec.~\ref{sec:exp-geo} will describe results for model estimation with geometric distances).

Five image pairs from the VGG dataset\footnote{http://www.robots.ox.ac.uk/~vgg/data/} (Corridor, House, Merton II, Wadham and Aerial View I) and four image pairs from the Zurich Building data set\footnote{\url{http://www.vision.ee.ethz.ch/showroom/zubud/}} (Building 04, Building 23, Building 36, Building 50 and Building 81) were used. The images were first resized before SIFT (as implemented on VLFeat~\cite{vedaldi2010vlfeat}) was used to extract around 500 correspondences per pair. To increase the outlier ratio, $20-30\%$ of the correspondences are randomly corrupted. An inlier threshold of $\epsilon = 1$ was used for all image pairs. For \ep~and \admm, apart from initialization with RANSAC and least squares, we also initialised it with $\ell_{\infty}$ outlier removal (variants \ep-$\ell_{\infty}$ and \admm-$\ell_{\infty}$). For all \ep~variants, the initial $\alpha$ was set to $0.5$ and $\kappa = 5$, while initial $\rho$ for all \admm~variants was set to $0.1$ and $\sigma=2.5$ for all the runs.

Table~\ref{table:fundamental_result} summarizes the quantitative results for all methods. Regardless of the initialization method, \ep~was able to find the largest consensus set. \admm~ variants converge to approximately the same solution quality as \ep~while taking slightly longer runtime. Fig.~\ref{fig:image_result} displays sample qualitative results for \ep. 

\subsubsection{Homography estimation (with algebraic error)}
\begin{figure*}[ht]\centering	
	\subfigure[Corridor.]{\includegraphics[width=0.3\textwidth]{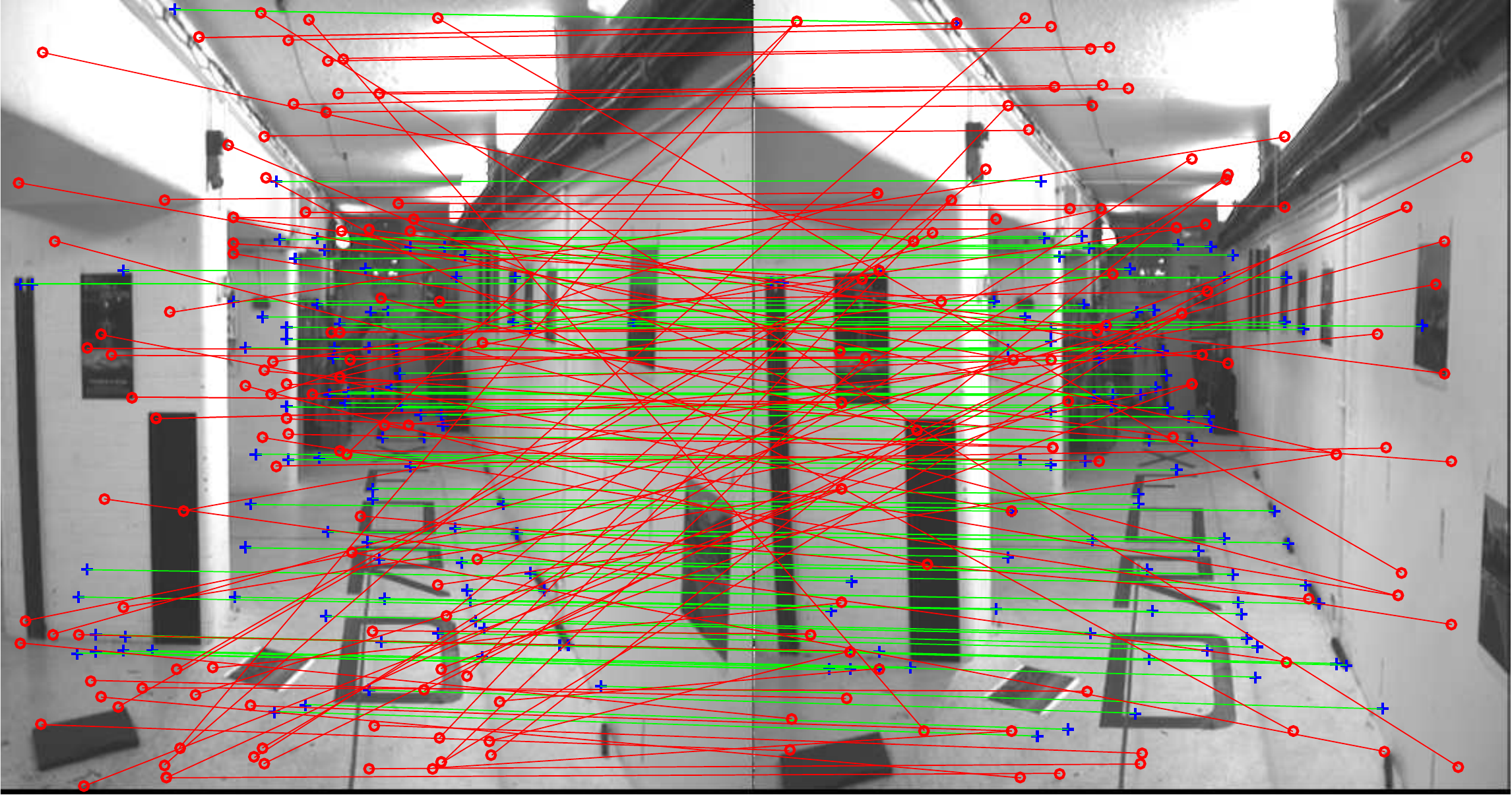}\label{subfig:fund_corridor}}
	\subfigure[House.]{\includegraphics[width=0.3\textwidth]{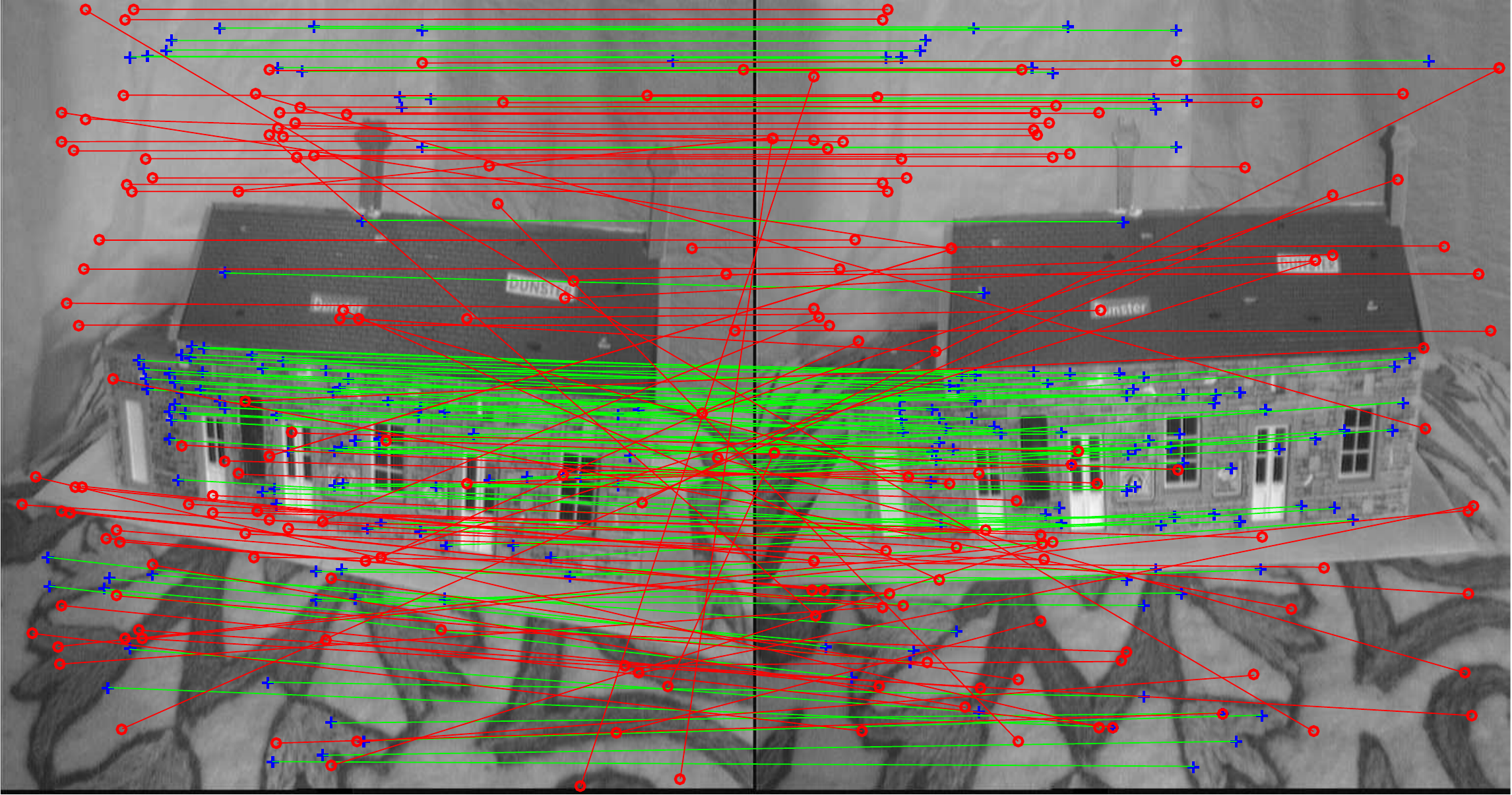}\label{subfig:fund_house}}
	\subfigure[Merton.]{\includegraphics[width=0.3\textwidth]{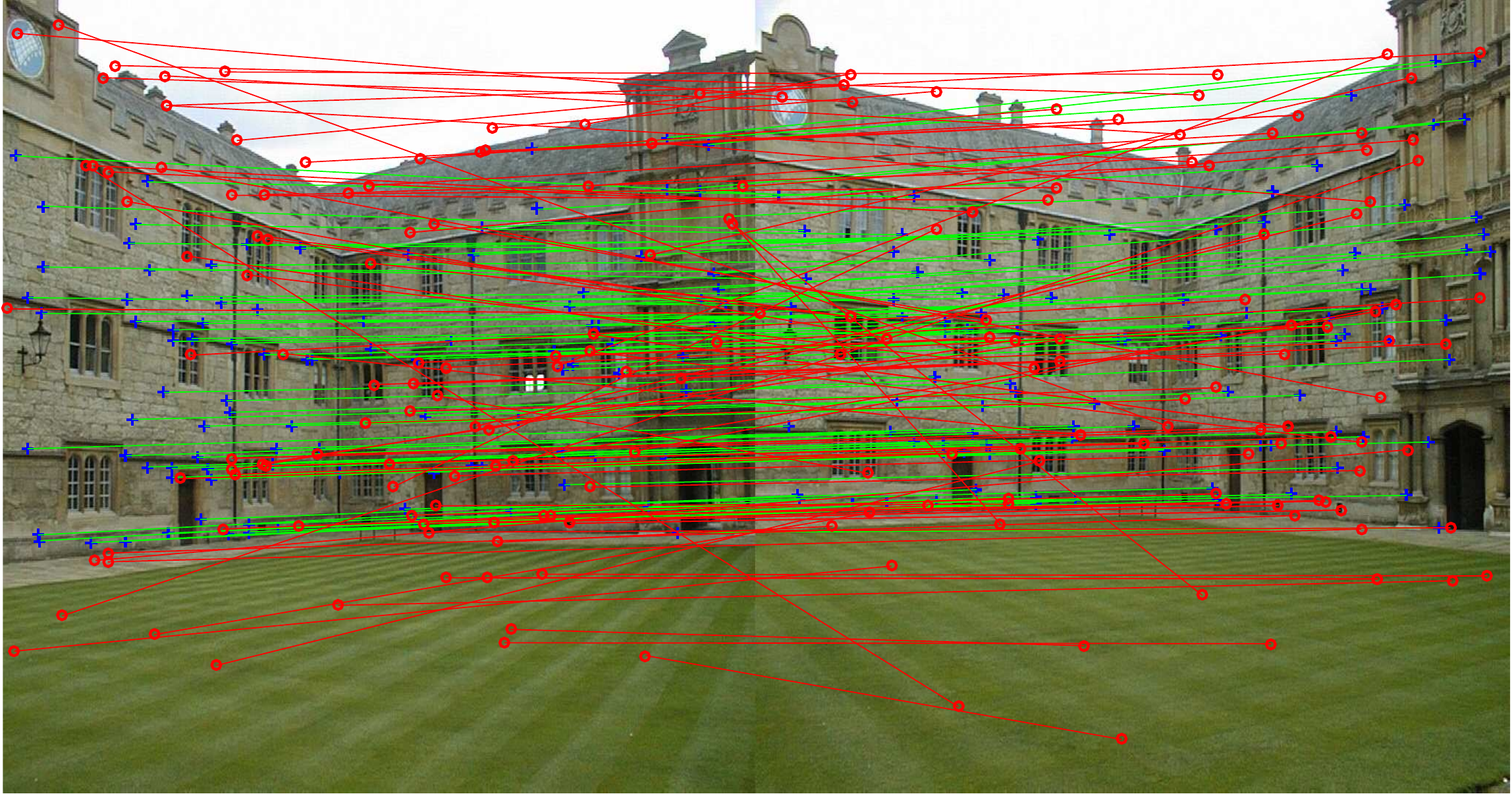}\label{subfig:fund_merton}}
	
	\subfigure[Union House. .]{\includegraphics[width=0.3\textwidth]{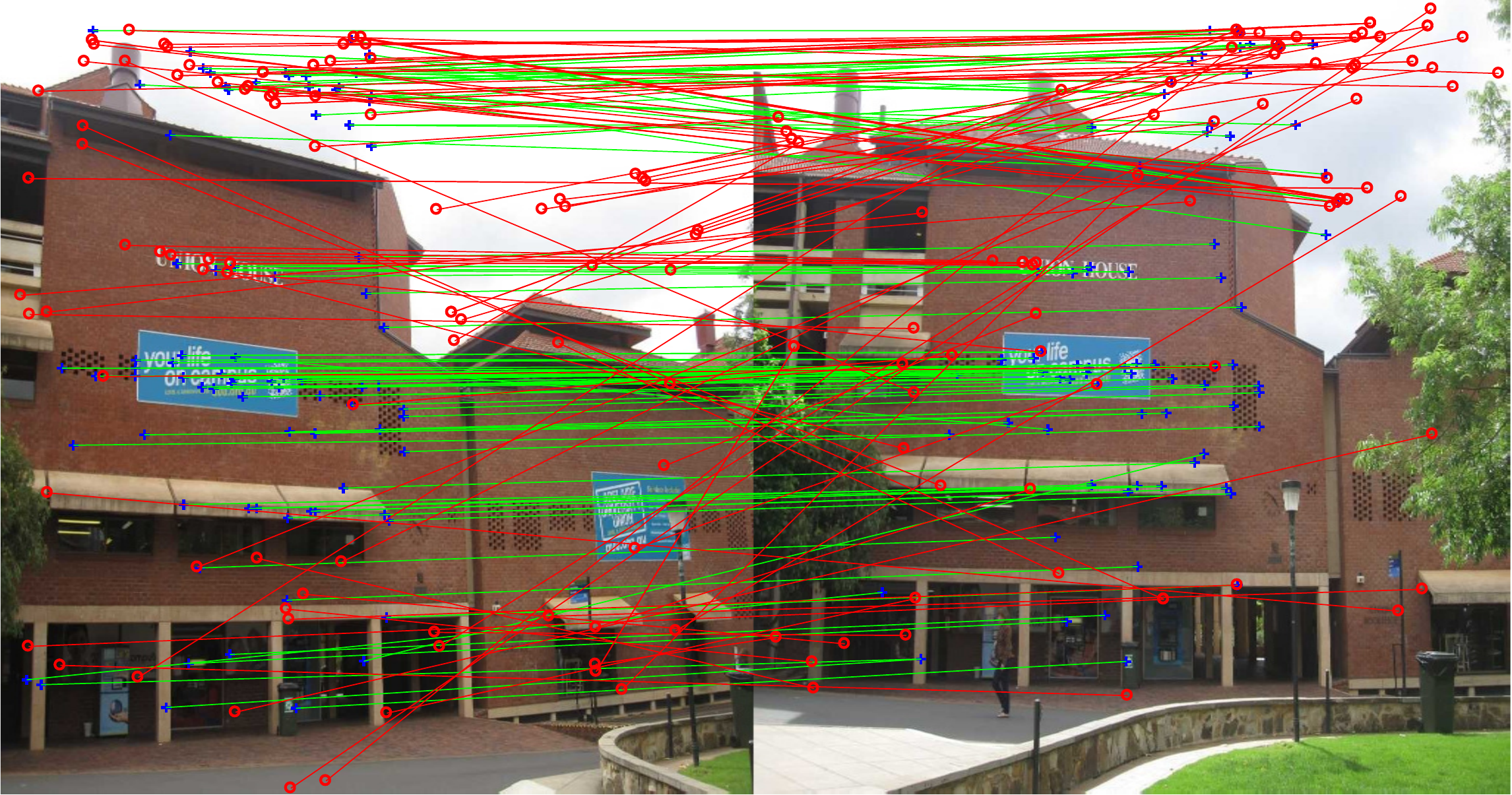}\label{subfig:linearhomo_unionhouse}}
	\subfigure[Building 64.]{\includegraphics[width=0.3\textwidth]{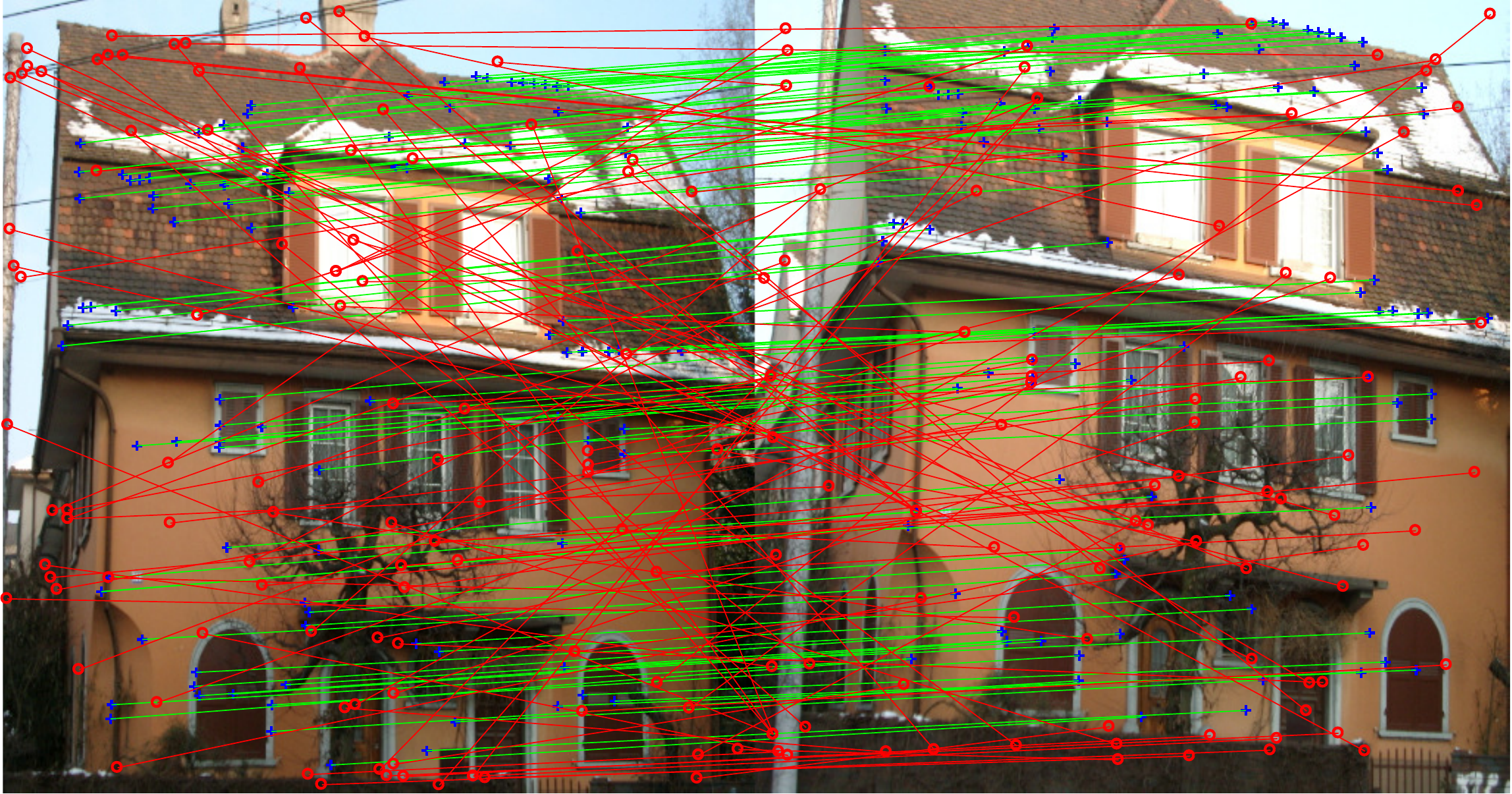}\label{subfig:linearhomo_building64}}
	\subfigure[Building 10.]{\includegraphics[width=0.3\textwidth]{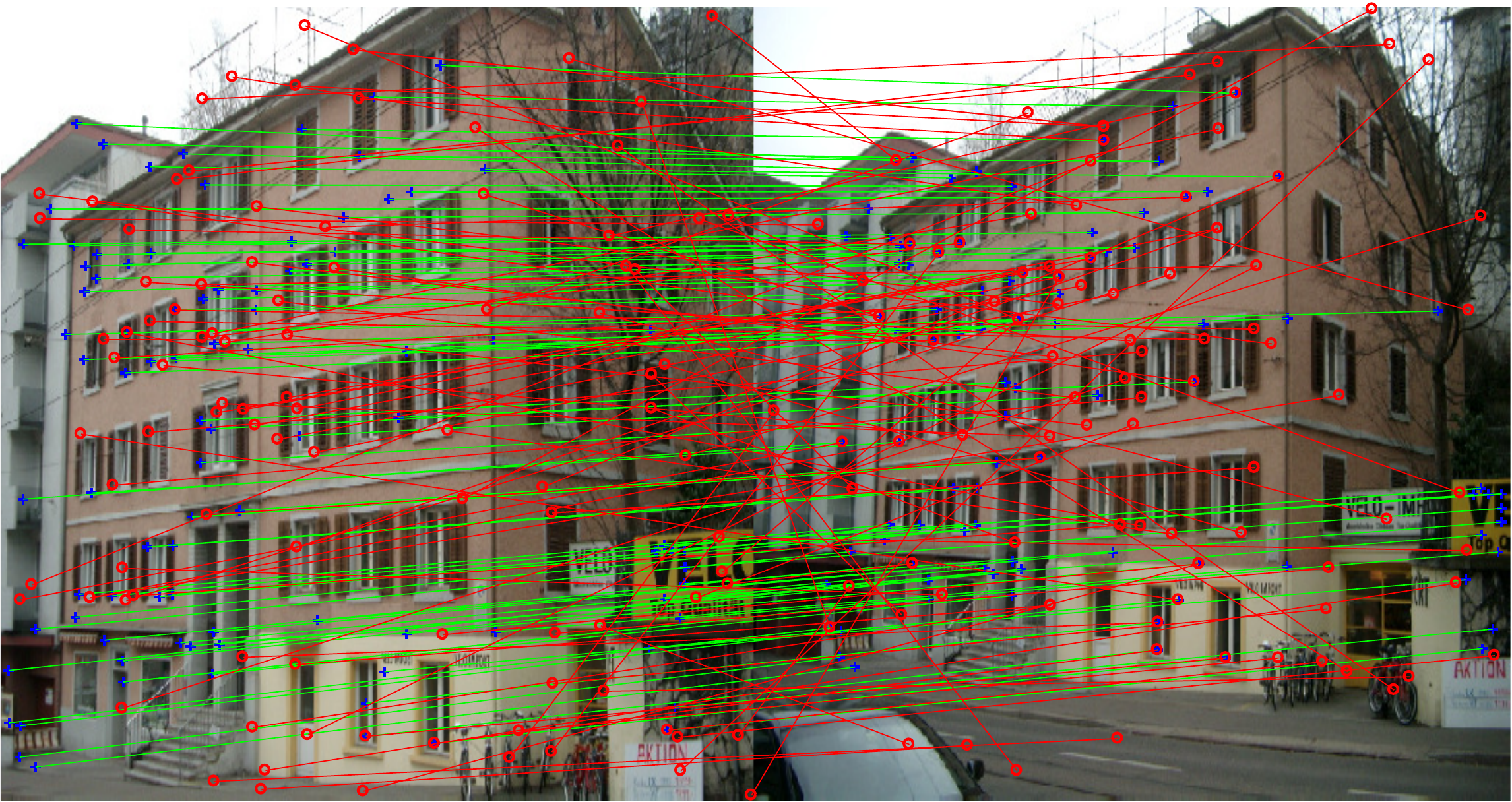}\label{subfig:linearhomo_building10}}
	
	\subfigure[Christ College Oxford.]{\includegraphics[width=0.3\textwidth]{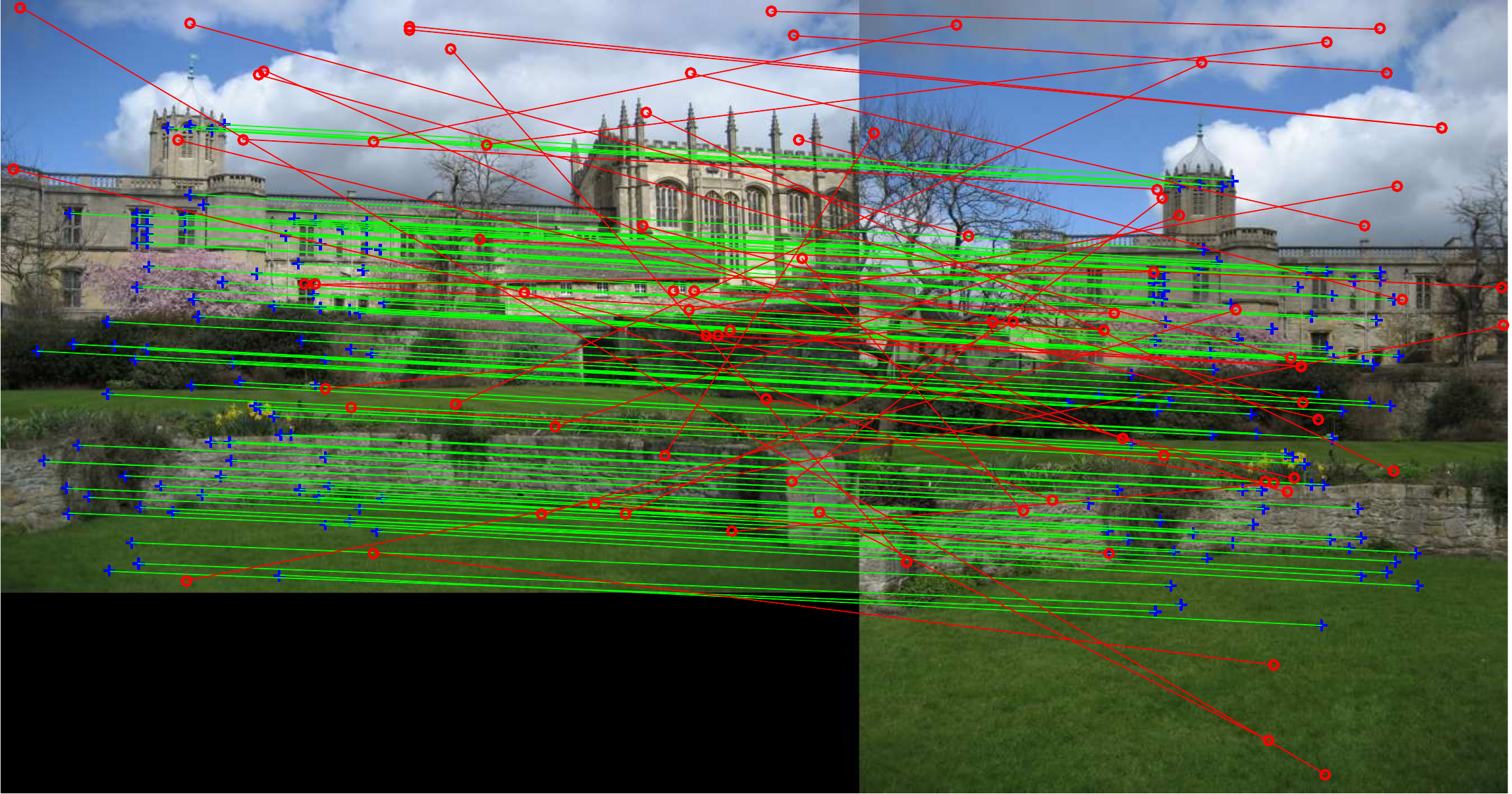}\label{subfig:homo_church}}
	\subfigure[Paris Invalides.]{\includegraphics[width=0.3\textwidth]{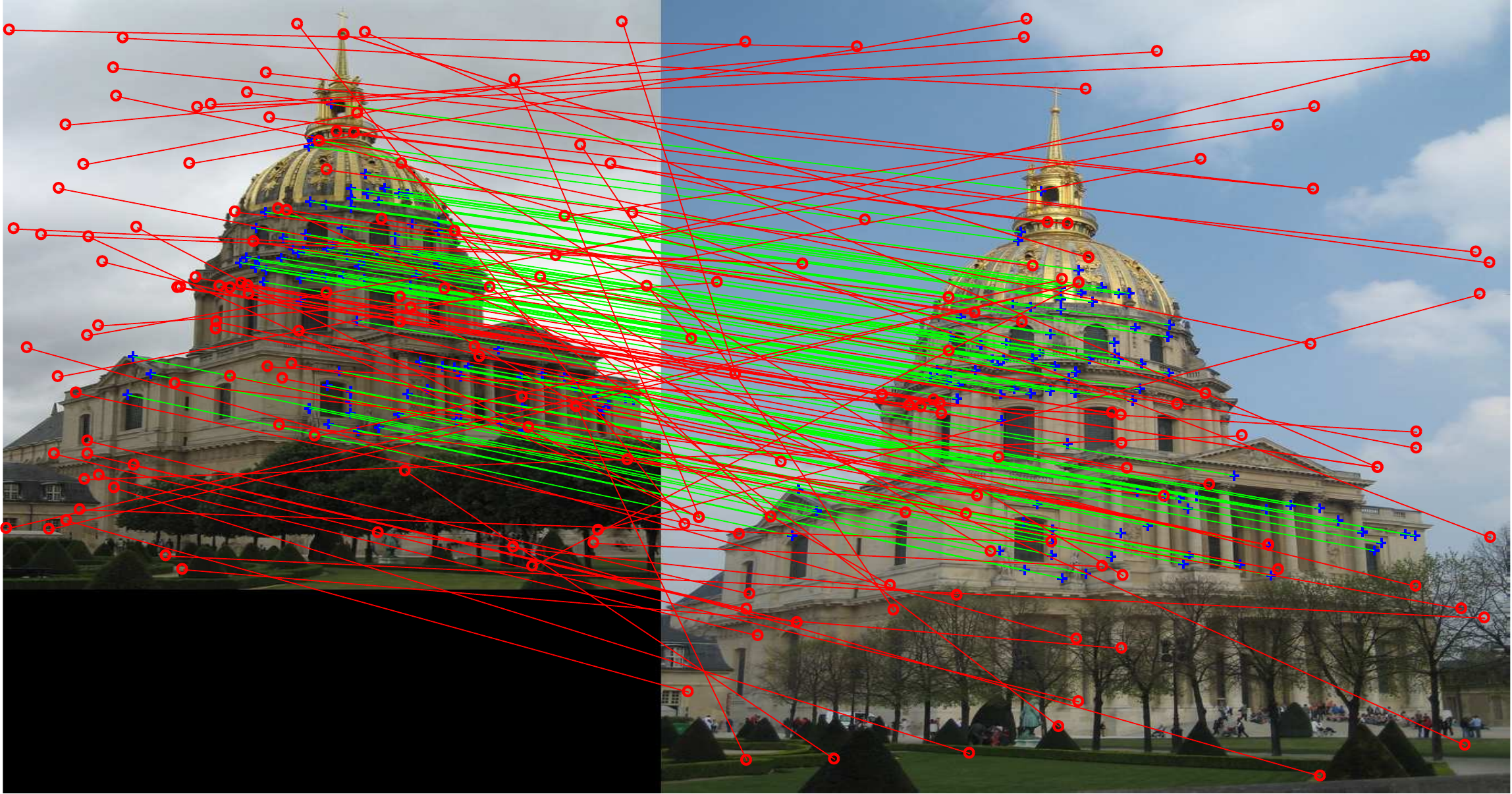}\label{subfig:homo_invalides}}
	\subfigure[University Library.]{\includegraphics[width=0.3\textwidth]{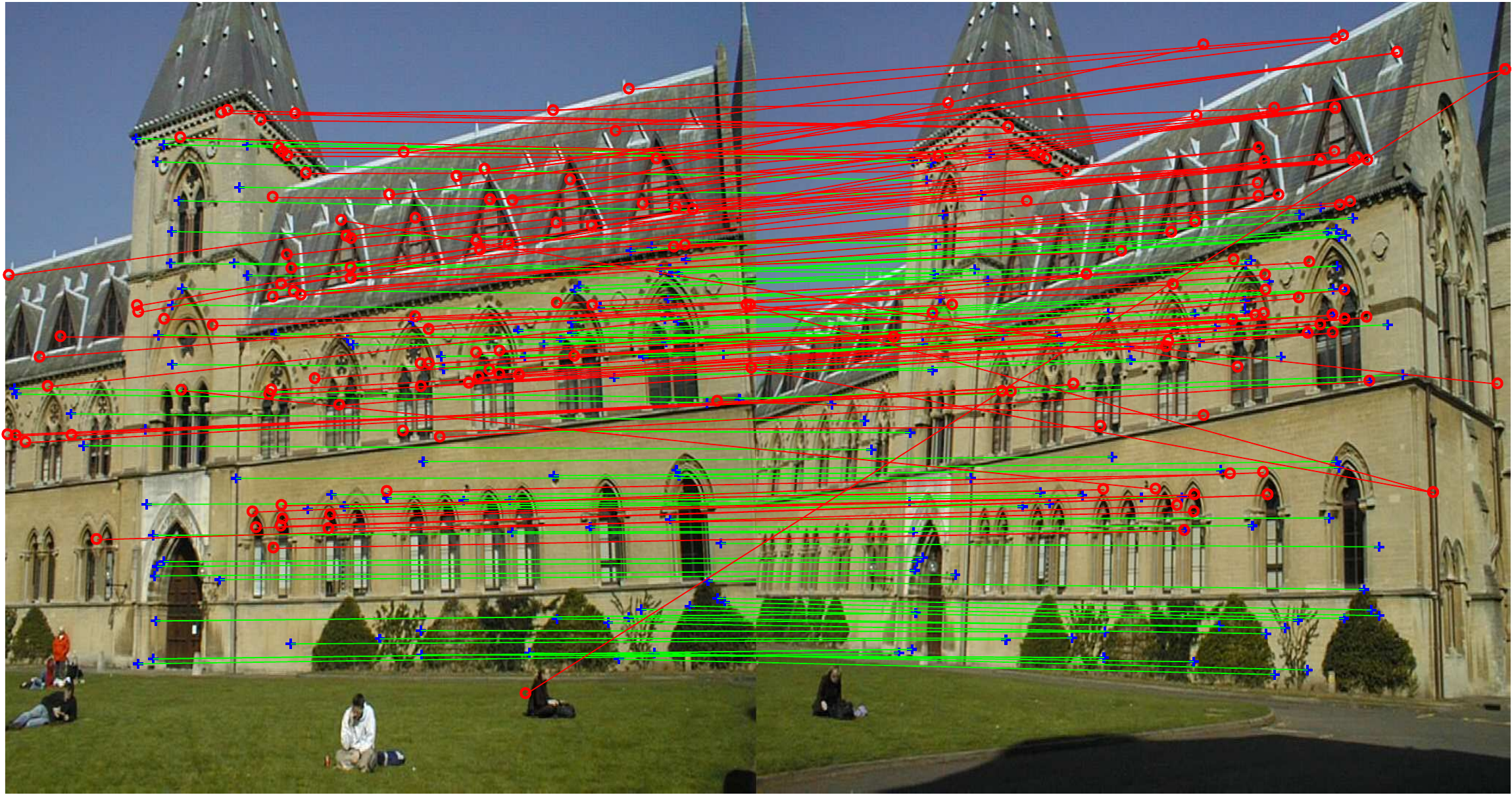}\label{subfig:homo_unilib}}
	\subfigure[Trees.]{\includegraphics[width=0.3\textwidth]{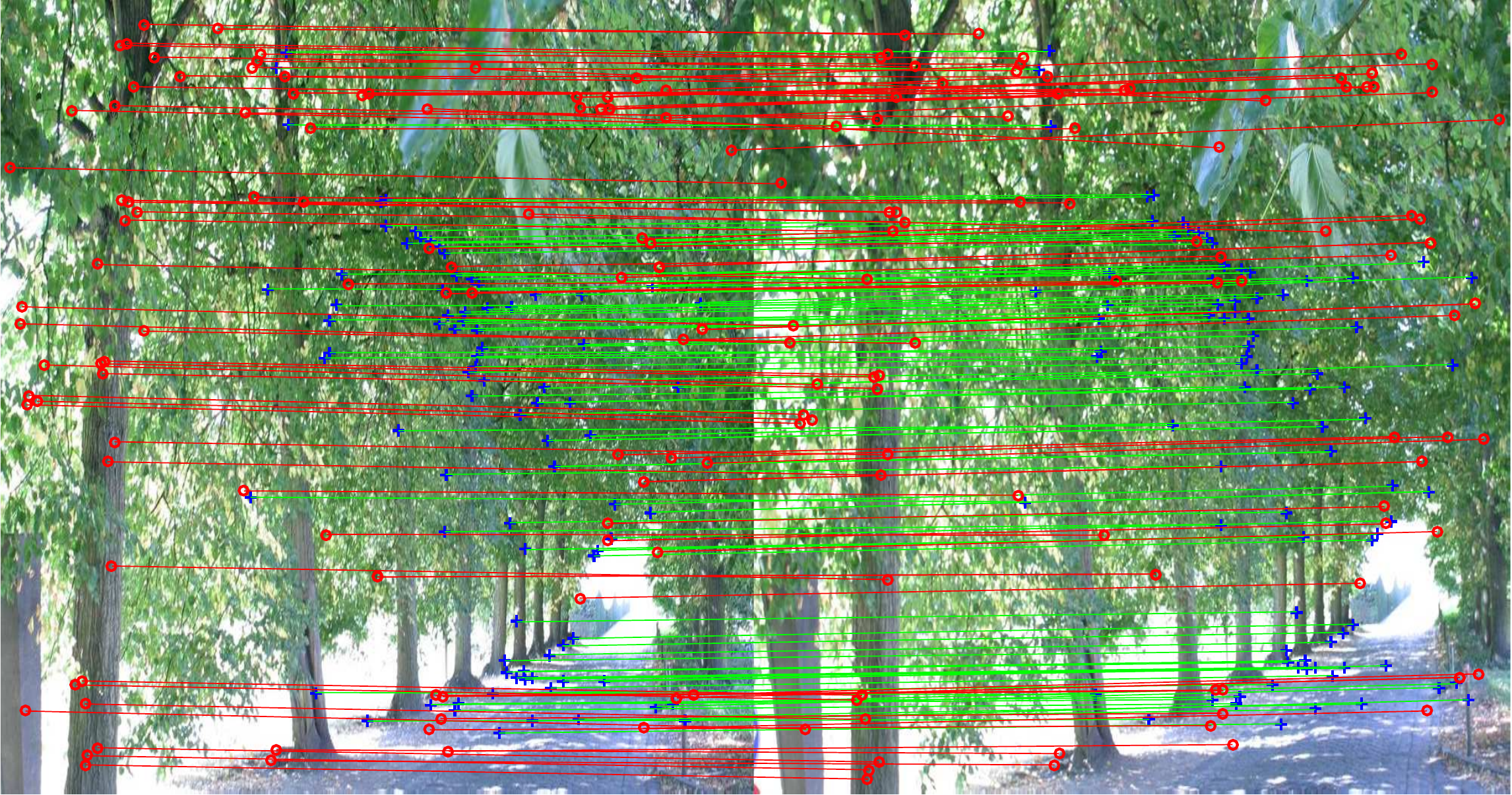}\label{subfig:affine_tree}}
	\subfigure[Boat.]{\includegraphics[width=0.3\textwidth]{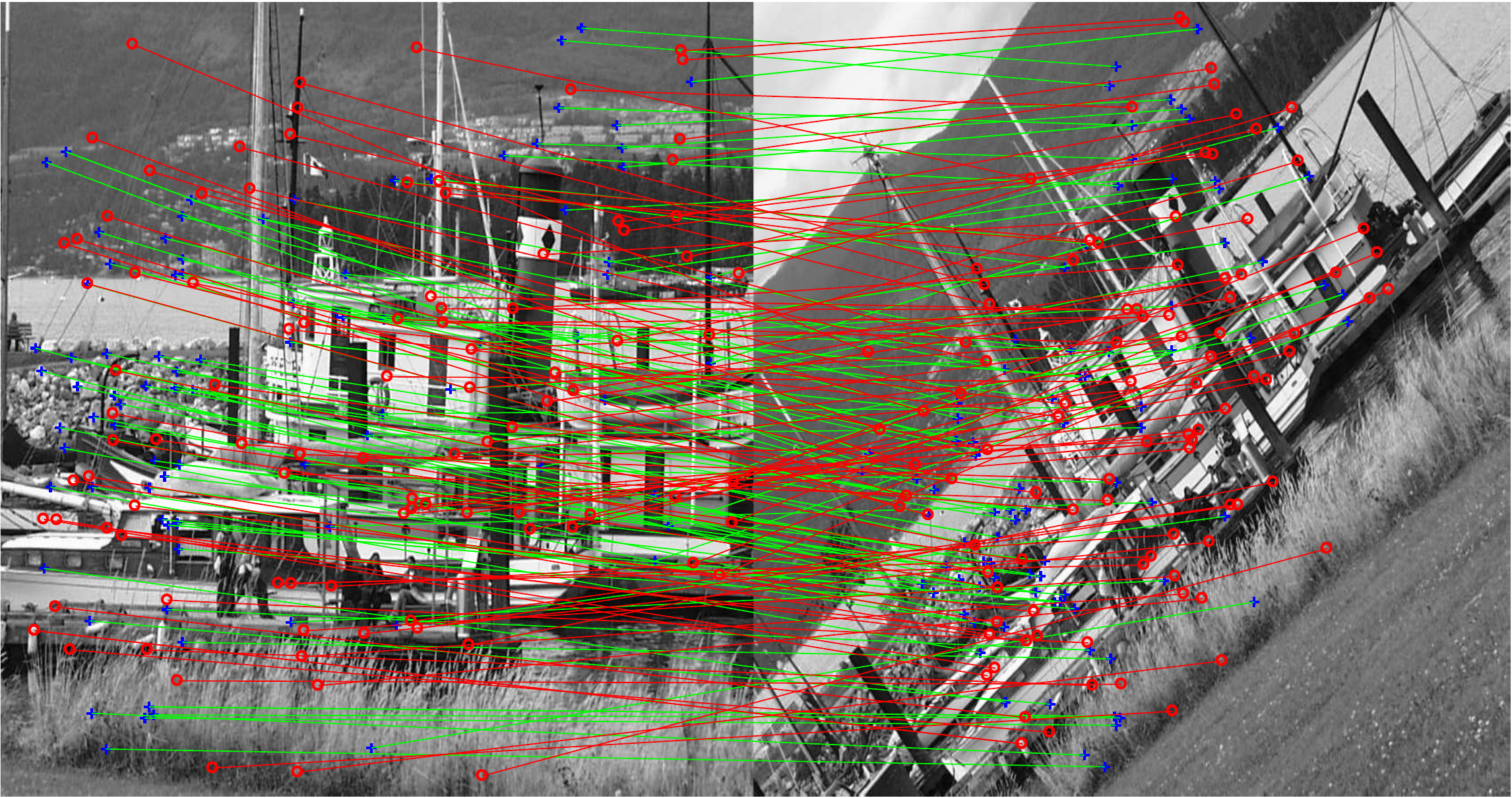}\label{subfig:affine_boat}}	
	\subfigure[Bark.]{\includegraphics[width=0.3\textwidth]{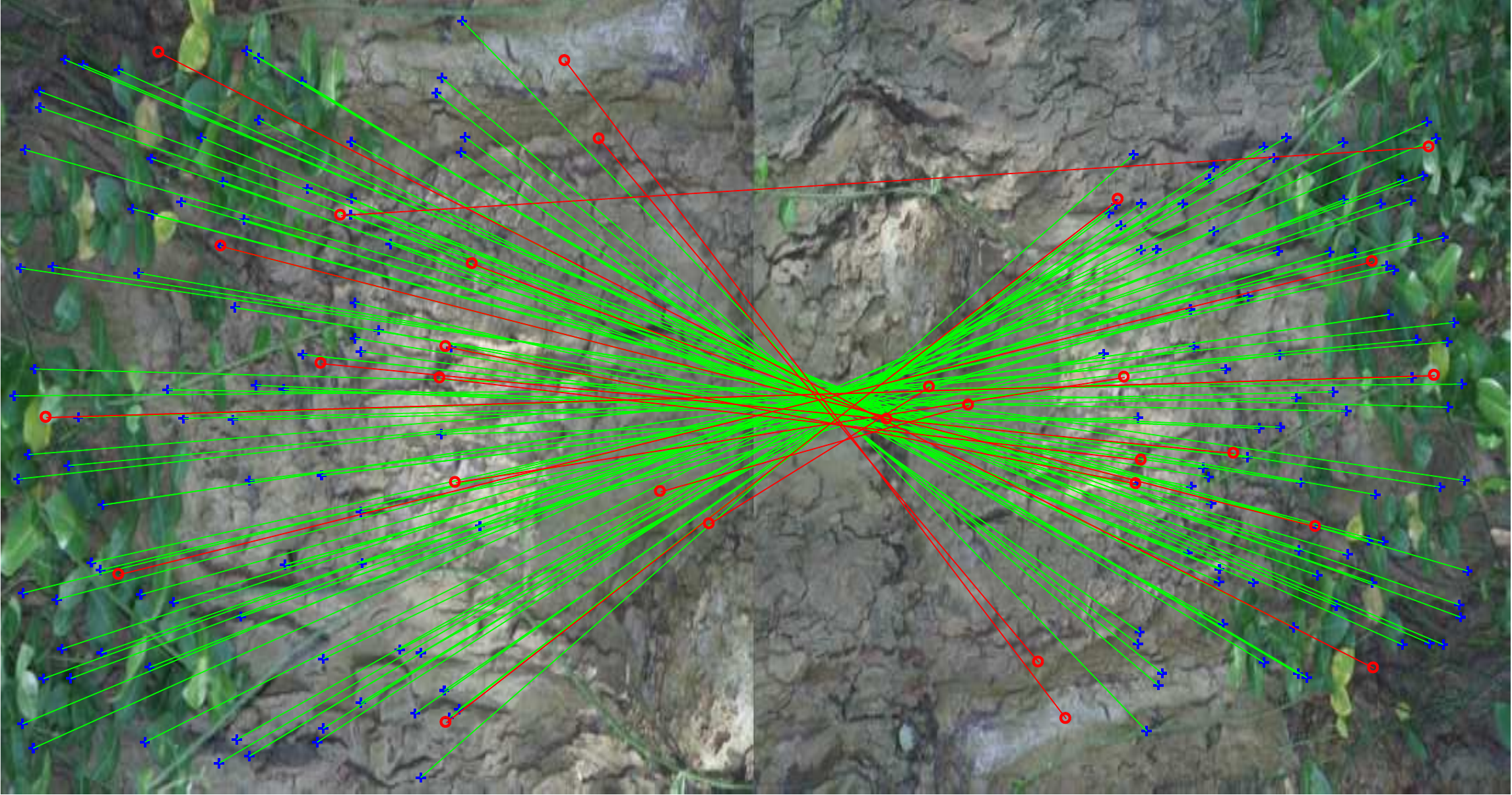}\label{subfig:affine_bark}}
	\caption{Qualitative results of local refinement methods on (a,b,c)  fundamental matrix estimation, (d,e,f) linearized homography estimation (g,h,i) homography estimation with geometric distance, and (j,k,l) affinity estimation. Green and red lines represent detected inliers and outliers. For clarity, only 100 inliers/outliers are plotted.}
	\label{fig:image_result}
\end{figure*}
Following~\cite[Chapter 4]{hartley2003multiple}, the homography constraints were linearized to investigate the performance of our algorithms. Five image pairs form the VGG dataset: University Library, Christ Church, Valbonne, Kapel and Paris's Invalides; three image pairs from the AdelaideRMF dataset~\cite{wong2011dynamic}: Union House, Old Classic Wing, Ball Hall and three pairs from the Zurich Building dataset: Building 64, Building 10 and Building 15 were used for this experiment. Parameters for the~\ep~and ~\admm~ variants were reused from the fundamental matrix experiment. Quantitative results displayed in Table~\ref{table:linear_homography_result} show that all the ~\ep~and ~\admm~ variants were able to achieve the highest consensus size.

\subsection{Models with geometric distances}\label{sec:exp-geo}

\subsubsection{Homography estimation}
We estimated 2D homographies based on the transfer error using all the methods. In the context of~\eqref{quasi-convex}, the geometric residual for homography (with $p=1$) is
\begin{equation}
\label{eq_homography}
\frac{\|(\btheta_{1:2} - \bv_i\btheta_3)\tilde{\bu_i}\|_1}{\btheta_3\tilde{\bu_i}},
\end{equation}
where $\btheta_{1:2}$ and $\btheta_3$ denote the first-two rows and the last row of the homography matrix, respectively. Each pair $(\bu_i,\bv_i)$ represents a point match across two views, and $\tilde{\bu_i} = [\bu^T 1]^T$. The data used in the linearized homography experiment was reused. The inlier threshold of $\epsilon = 4$ pixels was used for all input data. Initial $\alpha$  was set to $10$ and $\kappa=1.5$ for all \ep~variants. For \admm~variants, initial $\rho$ was set to $0.1$ and the increment rate $\sigma$ was set to $1.5$ for all the runs.

Quantitative results are shown in Table~\ref{table:homography_results}, and a sample qualitative result for \ep~is shown in Fig.~\ref{fig:image_result}. Similar to the fundamental matrix case, the \ep~variants outperformed the other methods in terms of solution quality, but were slower though its runtime was still within the same order of magnitude. \admm~variants also attain approximately the same solution as ~\ep~with slightly longer runtimes. Note that \ep-LSQ and \admm-LSQ were not invoked here, since finding least squares estimates based on geometric distances is intractable in general~\cite{hartley-accv07}.

%\vspace{-1em}
\subsubsection{Affinity estimation}
The previous experiment was repeated for affinity (6 DoF affine transformation) estimation, where the geometric matching error for the $i$-th correspondence can be written as:
\begin{equation}
\label{eq_affine}
\|\bu_i - \btheta\tilde{\bv}_i\|_1,
\end{equation} 
where each pair $(\bu_i,\bv_i)$ is a correspondence across two views, $\btheta \in \bbR^{2\times 3}$ represents the affine transformation, and $\tilde{\bv}_i = [\bv^T \;1]^T$. Initial $\alpha$ was set to $0.5$, $\kappa=5$ for~\ep~ variants and initial $\rho=0.5$ and $\sigma = 2.5$ for~\admm~variants. The inlier threshold was set to $\epsilon=2$ pixels. Five image pairs from VGG's affine image dataset: Bikes, Graff, Bark, Tree, Boat and five pairs of building from the Zurich Building Dataset: Building 143, Building 152, Building 163, Building 170 and Building 174 were selected for the experiment. Quantitative results are given in Table~\ref{table:affine}, and sample qualitative result is shown in Fig.~\ref{fig:image_result}. Similar conclusions can be drawn.
%\vspace{-1em}

\subsubsection{Triangulation}
\input{triangulation_results_summ.tex}

\begin{figure}[ht]\centering
	\includegraphics[width=0.45\columnwidth]{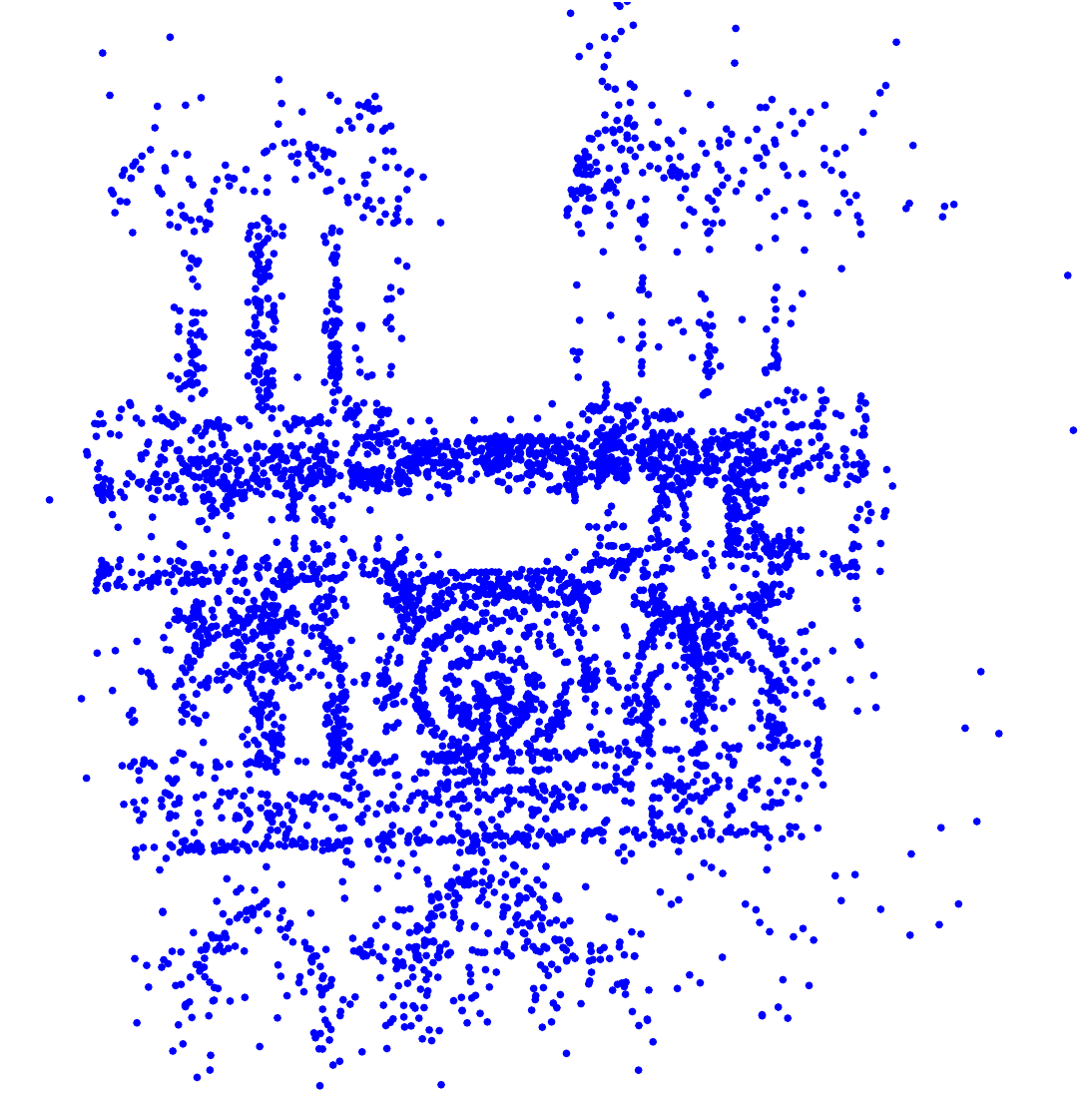}
	\caption{Qualitative results of EP-RS on triangulation.}
	\label{fig:notredame_triangulated}
\end{figure}
We conducted triangulation from outlier-contaminated multiple-view observations of 3D points. For each image point $\bx_i$ and the camera matrix $\bP_i \in \bbR^{3\times 4}$, the following re-projection error with respect to the point estimation $\btheta$ was used in our experiments:
\begin{equation}
\label{eq_triangulation}
\frac{\|(\bP^i_{1:2} - \bx_i\bP^i_3)\btheta^\prime\|_1}{\bP^i_3\tilde{\btheta}},
\end{equation}
where $\tilde{\btheta} = [\btheta^T\; 1]^T$, $\bP^i_{1:2}$ denotes the first two rows of the camera matrix and $\bP^i_3$ represents its third row. We selected five feature tracks from the NotreDame dataset~\cite{snavely2006photo} with more than $N = 150$ views each to test our algorithm. The inlier threshold for maximum consensus was set to $\epsilon = 1$ pixel.  $\alpha$ was initially set to 0.5 and $\kappa = 1.5$ for all variants of \ep. For the~\admm~variants, initial $\rho$ was set to $0.1$ and $\sigma = 2.5$. Table~\ref{table:triangulation_results} shows the quantitative results. Again, the variants of local refinement algorithms are better than the other methods in terms of solution quality. The runtime gap was not as significant here due to the low-dimensionality of the model ($d = 3$).
We repeated the experiments for all 11595 feature tracks in the dataset with more than 10 views. All the methods were executed with $\epsilon=1$ pixel and the same set of parameters. Table~\ref{table:triangulation_results_all} lists the total number of inliers and runtime for all the methods over all tested points. With RANSAC initialization, ~\ep-RS~ was able to achieve the highest total number of inliers followed by~\admm-RS. The triangulated result is shown in Figure~\ref{fig:notredame_triangulated}.
\section{Conclusions}
We introduced two novel deterministic approximate algorithms for maximum consensus, based on non-smooth penalized method and ADMM. In terms of solution quality, our algorithms outperform other heuristic and approximate methods---this was demonstrated particularly by our methods being able to improve upon the solution of RANSAC. Even when presented with bad initializations (i.e., when using least squares to initialize on unbalanced data), our methods was able to recover and attain good solutions. Though our methods can be slower, their runtimes are still well within practical range (seconds to tens of seconds). In fact, at high outlier rates, our methods is actually faster than the RANSAC variants, while yielding higher-quality results. 
Overall, the experiments illustrate that the proposed method can serve well in settings where slight additional runtime is a worthwhile expense for guaranteed convergence to an improved maximum consensus solution. 
\begin{comment}
\appendices
\section{}
Appendix two text goes here.
\end{comment}
% use section* for acknowledgment
\ifCLASSOPTIONcompsoc
  % The Computer Society usually uses the plural form
  \section*{Acknowledgments}
\else
  % regular IEEE prefers the singular form
  \section*{Acknowledgment}
\fi
Chin and Suter were supported by DP160103490. Eriksson was supported by FT170100072. The authors are also grateful to CE140100016 which funded this collaborative work.
%\vfill
% Can use something like this to put references on a page
% by themselves when using endfloat and the captionsoff option.
%\ifCLASSOPTIONcaptionsoff
%  \newpage
%\fi
% trigger a \newpage just before the given reference
% number - used to balance the columns on the last page
% adjust value as needed - may need to be readjusted if
% the document is modified later
%\IEEEtriggeratref{8}
% The "triggered" command can be changed if desired:
%\IEEEtriggercmd{\enlargethispage{-5in}}
% references section
% can use a bibliography generated by BibTeX as a .bbl file
% BibTeX documentation can be easily obtained at:
% http://mirror.ctan.org/biblio/bibtex/contrib/doc/
% The IEEEtran BibTeX style support page is at:
% http://www.michaelshell.org/tex/ieeetran/bibtex/
\bibliographystyle{IEEEtran}
% argument is your BibTeX string definitions and bibliography database(s)
\bibliography{localmaxcon}

\vfill
% insert where needed to balance the two columns on the last page with
% biographies
%\newpage
% You can push biographies down or up by placing
% a \vfill before or after them. The appropriate
% use of \vfill depends on what kind of text is
% on the last page and whether or not the columns
% are being equalized.

% Can be used to pull up biographies so that the bottom of the last one
% is flush with the other column.
%\enlargethispage{-5in}
% that's all folks
\end{document}

%% file: fundamental_results.tex
\begin{table*}[ht]
  \centering
  \caption{Fundamental matrix estimation results (with algebraic error)}  
  \resizebox{\linewidth}{!}{
	  \begin{tabular}{|c|c|c|c|c|c|c|c|c|c|c|c|c|c|c|}
	    \hline
	    \multicolumn{2}{|c|}{Methods} & RS   & PS   & GMLE & LORS & CRS & $\ell_1$   & $\ell_{\infty}$ & EP-RS        & EP-LSQ       & EP-$\ell_{\infty}$ & AM-RS & AM-LSQ & AM-$\ell_{\infty}$ \\ \hline
	    \multirow{2}{*}{\begin{tabular}[c]{@{}c@{}}House\\ N = 556\end{tabular}}    & $|\cI|$       & 240 & 245 & 252 & 265 & 267 & 115 & 175 & \textbf{275} & \textbf{275} & \textbf{275} & \textbf{275} & 267 & \textbf{275}\\ %\cline{2-12} 
	                                  & time (s)  & 1.33 & 1.07 & 1.01 & 0.99 & 0.75 & 0.2 & 0.1 & 2.05 & 1.75 & 2.32 & 6.35 & 7.13 & 6.15  \\ \hline
	    \multirow{2}{*}{\begin{tabular}[c]{@{}c@{}}Aerial\\ N = 483\end{tabular}}   & $|\cI|$        & 264 & 265 & 260 & 280 & 287 & 213 & 221 & 290 & 290 & 290 & 295 & 295 & \textbf{300} \\ %\cline{2-12} 
	                                  & time (s) & 0.53 & 0.46 & 0.55 & 0.35 & 0.37 & 0.13 & 0.15 & 1.15 & 0.95 & 1.13 & 4.75 & 6.25 & 7.12   \\ \hline
	    \multirow{2}{*}{\begin{tabular}[c]{@{}c@{}}Merton \\ N = 590\end{tabular}}  & $|\cI|$        & 295 & 295 & 301 & 306 & 306 & 147 & 227 & \textbf{321} & \textbf{321} & \textbf{321} & 307 & 305 & 302 \\ %\cline{2-12} 
	                                  & time (s) & 0.65 & 0.25 & 0.30 & 0.25 & 0.30 & 0.25 & 0.13 & 1.18 & 0.95 & 1.05 & 5.15 & 5.73 & 5.78        \\ \hline
	    \multirow{2}{*}{\begin{tabular}[c]{@{}c@{}}Wadham\\ N = 618\end{tabular}}   & $|\cI|$        & 305 & 307 & 315 & 320 & 325 & 271 & 290 & \textbf{330} & \textbf{330} & \textbf{330} & 310 & \textbf{330} & 315 \\ %\cline{2-12} 
	                                  & time (s) & 1.52 & 1.35 & 1.15 & 1.05 & 1.08 & 0.15 & 0.27 & 2.25 & 1.41 & 1.42 & 8.88 & 7.51 & 6.52     \\ \hline
	    \multirow{2}{*}{\begin{tabular}[c]{@{}c@{}}Corridor\\ N = 684\end{tabular}} & $|\cI|$       & 310 & 310 & 315 & 327 & 330 & 251 & 300 & 375 & \textbf{390} & \textbf{390} & 388 & 375 & \textbf{390} \\ %\cline{2-12} 
	                                  & time (s)  & 0.95 & 1.12 & 0.97 & 0.65 & 0.75 & 0.15 & 0.27 & 2.35 & 1.17 & 1.26 & 6.52 & 5.56 & 7.0      \\ \hline
	    \multirow{2}{*}{\begin{tabular}[c]{@{}c@{}}Building 81\\ N = 525\end{tabular}} & $|\cI|$         & 262 & 267 & 251 & 270 & 277 & 115 & 212 & \textbf{315} & \textbf{315} & \textbf{315} & \textbf{315} & 300 & 300\\ %\cline{2-12} 
	                                  & time (s) & 1.15 & 1.07 & 1.12 & 0.95 & 0.89 & 0.11 & 0.17 & 1.95 & 0.99 & 1.17 & 5.25 & 6.69 & 2.45      \\ \hline
	    \multirow{2}{*}{\begin{tabular}[c]{@{}c@{}}Building 04\\ N = 394\end{tabular}} & $|\cI|$         & 181 & 180 & 175 & 190 & 192 & 97 & 171 & 197 & 197 & 197 & \textbf{200} & 122 & 184 \\ %\cline{2-12} 
	                                  & time (s)   & 1.21 & 1.25 & 1.19 & 1.05 & 2.17 & 0.17 & 0.15 & 2.47 & 1.13 & 1.06 & 10.67 & 7.89 & 9.2      \\ \hline
	    \multirow{2}{*}{\begin{tabular}[c]{@{}c@{}}Building 23\\ N = 699\end{tabular}} & $|\cI|$         & 315 & 308 & 305 & 328 & 327 & 250 & 259 & \textbf{330} & \textbf{330} & \textbf{330} & 323 & 123 & 316 \\ %\cline{2-12} 
	                                  & time (s) & 1.45 & 1.44 & 1.96 & 1.24 & 1.15 & 0.15 & 0.11 & 3.17 & 2.06 & 2.89 & 7.97 & 5.85 & 5.02      \\ \hline
	    \multirow{2}{*}{\begin{tabular}[c]{@{}c@{}}Building 36\\ N = 651\end{tabular}} & $|\cI|$         & 275 & 275 & 280 & 290 & 295 & 159 & 220 & \textbf{320} & \textbf{320} & \textbf{320} & 315 & \textbf{320} & 315 \\ %\cline{2-12} 
	                                  & time (s)  & 1.62 & 1.59 & 1.71 & 1.05 & 1.12 & 0.15 & 0.11 & 2.61 & 1.42 & 1.36 & 5.39 & 7.46 & 8.71      \\ \hline
	   % \multirow{2}{*}{\begin{tabular}[c]{@{}c@{}}Building 50\\ N = 365\end{tabular}} & $|\cI|$        & 320 & 307 & 319 & 322 & 322 & 42 & 317 & \textbf{325} & \textbf{325} & \textbf{325} & 322 & 324 & 323 \\ %\cline{2-12} 
	   %                               & time (s) & 0.05 & 0.05 & 0.04 & 0.16 & 0.24 & 0.02 & 0.01 & 0.6 & 0.55 & 0.53 & 3.78 & 4.54 & 3.28     \\ \hline
	  \end{tabular}
  }
   \label{table:fundamental_result}
\end{table*}

%% file: linear_homography_results.tex
\begin{table*}[ht]
  \centering
  \caption{Homography estimation results (with algebraic error)}    
  \resizebox{\linewidth}{!}{
	  \begin{tabular}{|c|c|c|c|c|c|c|c|c|c|c|c|c|c|c|}
	    \hline
	    \multicolumn{2}{|c|}{Methods} & RS   & PS   & GMLE & LORS & CRS & $\ell_1$   & $\ell_{\infty}$ & EP-RS        & EP-LSQ       & EP-$\ell_{\infty}$ & AM-RS & AM-LSQ & AM-$\ell_{\infty}$ \\ \hline
	    \multirow{2}{*}{\begin{tabular}[c]{@{}c@{}}University Library\\ N = 439\end{tabular}}    & $|\cI|$       & 220 & 221 & 215 & 230 & 229 & 157 & 191 & \textbf{295} & \textbf{295} & \textbf{295} & 280 & 290 & \textbf{295} \\ %\cline{2-12} 
	                                  & time (s) & 1.15 & 1.27 & 1.05 & 1.02 & 0.97 & 0.15 & 0.25 & 2.79 & 1.09 & 0.97 & 9.19 & 14.25 & 7.81         \\ \hline
	    \multirow{2}{*}{\begin{tabular}[c]{@{}c@{}}Christ Church\\ N = 524\end{tabular}}   & $|\cI|$       & 259 & 262 & 265 & 273 & 277 & 267 & 251 & 315 & 315 & 315 & \textbf{317} & 311 & \textbf{315} \\ %\cline{2-12} 
	                                  & time (s)  & 1.15 & 1.12 & 1.01 & 1.19 & 1.05 & 0.09 & 0.15 & 2.99 & 1.78 & 1.91 & 9.79 & 8.46 & 15.21
	                                     \\ \hline
	    \multirow{2}{*}{\begin{tabular}[c]{@{}c@{}}Kapel\\ N = 449 \end{tabular}}  & $|\cI|$        & 156 & 155 & 162 & 165 & 160 & 95 & 115 & \textbf{210} & \textbf{210} & \textbf{210} & 200 & 201 & 205 \\ %\cline{2-12} 
	                                  & time (s) & 1.18 & 1.12 & 1.18 & 1.44 & 1.65 & 0.11 & 0.07 & 2.22 & 1.32 & 1.29 & 10.41 & 9.74 & 11.01        \\ \hline
	                                  
		 \multirow{2}{*}{\begin{tabular}[c]{@{}c@{}}Invalides\\ N = 558\end{tabular}} & $|\cI|$         & 178 & 170 & 169 & 180 & 185 & 117 & 107 & 230 & 230 & 230 & \textbf{231} & 229 & 229 \\ %\cline{2-12} 
		& time (s)  & 2.01 & 2.76 & 1.79 & 1.85 & 1.55 & 0.09 & 0.07 & 3.35 & 3.01 & 4.15 & 10.2 & 9.81 & 10.47     \\ \hline                                  
		
	    \multirow{2}{*}{\begin{tabular}[c]{@{}c@{}}Union House\\N = 520\end{tabular}}   & $|\cI|$       & 221 & 225 & 227 & 220 & 230 & 185 & 210 & \textbf{290} & \textbf{290} & \textbf{290} & \textbf{290} & \textbf{290} & 287\\ %\cline{2-12} 
	                                  & time (s) & 1.16 & 1.16 & 1.05 & 1.09 & 1.08 & 0.07 & 0.05 & 2.4 & 1.43 & 1.23 & 7.41 & 8.23 & 8.85      \\ \hline
	    \multirow{2}{*}{\begin{tabular}[c]{@{}c@{}}Old Classic Wing\\ N = 561\end{tabular}} & $|\cI|$       & 206 & 206 & 211 & 215 & 214 & 181 & 187 & \textbf{250} & \textbf{250} & \textbf{250} & 229 & \textbf{250} & \textbf{250} \\ %\cline{2-12} 
	                                  & time (s)  & 1.95 & 1.86 & 1.88 & 1.15 & 1.10 & 0.07 & 0.07 & 2.19 & 1.14 & 1.27 & 6.36 & 3.35 & 5.51     \\ \hline
	    \multirow{2}{*}{\begin{tabular}[c]{@{}c@{}}Ball Hall\\ N = 538\end{tabular}} & $|\cI|$        & 170 & 177 & 175 & 188 & 182 & 110 & 187 & \textbf{215} & \textbf{215} & \textbf{215} & 209 & 202 & 200\\ %\cline{2-12} 
	                                  & time (s) & 1.85 & 1.77 & 1.16 & 1.53 & 1.43 & 0.04 & 0.06 & 3.39 & 2.27 & 2.78 & 9.64 & 7.47 & 10.74     \\ \hline
	    \multirow{2}{*}{\begin{tabular}[c]{@{}c@{}}Building 64\\ N = 529\end{tabular}} & $|\cI|$         & 185 & 187 & 184 & 190 & 197 & 100 & 112 & \textbf{233} & \textbf{233} & \textbf{233} & 216 & 211 & 215 \\ %\cline{2-12} 
	                                  & time (s)    & 1.75 & 1.56 & 1.22 & 1.56 & 0.99 & 0.09 & 0.05 & 2.86 & 1.49 & 2.01 & 6.44 & 8.61 & 6.87      \\ \hline
	    \multirow{2}{*}{\begin{tabular}[c]{@{}c@{}}Building 10\\ N = 546\end{tabular}} & $|\cI|$       & 210 & 215 & 217 & 222 & 227 & 191 & 178 & 250 & 250 & 250 & \textbf{251} & 250 & 250 \\ %\cline{2-12} 
	                                  & time (s) & 0.09 & 0.12 & 0.1 & 0.31 & 0.43 & 0.06 & 0.05 & 4.14 & 4.08 & 4.15 & 8.56 & 8.1 & 8.92     \\ \hline
%	    \multirow{2}{*}{\begin{tabular}[c]{@{}c@{}}Building 15\\ N = 477\end{tabular}} & $|\cI|$       & 404 & 398 & 392 & 405 & 404 & 379 & 364 & 405 & 405 & 405 & \textbf{406} & 392 & 388 \\ %\cline{2-12} 
	                                  %& time (s)& 0.06 & 0.06 & 0.05 & 1.88 & 2.33 & 0.07 & 0.05 & 4.55 & 4.69 & 4.63 & 7.32 & 7.66 & 7.75   \\ \hline
	  \end{tabular}
  }
   \label{table:linear_homography_result}
\end{table*}

%% file: homography_results.tex
\begin{table*}
	\centering
	\caption{Homography estimation results (with geometric transfer error)}
	\resizebox{\linewidth}{!}{
		\begin{tabular}{|c|c|c|c|c|c|c|c|c|c|c|c|c|}
			\hline
			\multicolumn{2}{|c|}{Methods} & RS & PS & GMLE & LORS & LORS1 & $\ell_1$ & $\ell_\infty$ & EP-RS & EP-$\ell_\infty$ & AM-RS & AM-$\ell_\infty$ \\ \hline		
			\multirow{2}{*}{\begin{tabular}[c]{@{}c@{}}University Library\\ N = 439\end{tabular}} & $|\cI|$& 136 & 150 & 149 & 155 & 157 & 97 & 86 & \textbf{210} & \textbf{210} & 195 & 205 	\\	 
			& time (s) & 2.53 & 2.451 & 2.41 & 2.52 & 2.41 & 1.06 & 1.65 & 7.53 & 5.32 & 10.95 & 9.85 \\ \hline   
			
			\multirow{2}{*}{\begin{tabular}[c]{@{}c@{}}Christ Church\\ N = 539\end{tabular}} & $|\cI|$ & 125 & 127 & 130 & 125 & 129 &  101 & 120 & \textbf{186} & \textbf{186} & 175 & \textbf{186} 	\\	 
			& time (s)  & 2.79 & 2.52 & 2.5 & 2.44 & 2.53 & 1.35 & 2.09 & 8.95 & 6.93 & 16.82 & 18.16 \\ \hline   
			
			\multirow{2}{*}{\begin{tabular}[c]{@{}c@{}}Kapel\\ N = 543 \end{tabular}} & $|\cI|$ & 160 & 167 & 160 & 160 & 157 & 110 & 104 & \textbf{175} & \textbf{175} & 169 & 168 \\	 
			& time (s)  & 2.84 & 2.11 & 3.87 & 2.31 & 2.68 & 2.7 & 2.07 & 7.44 & 9.32 & 13.17 & 11.61\\ \hline 		
		
			\multirow{2}{*}{\begin{tabular}[c]{@{}c@{}}Invalides\\ N = 558 \end{tabular}} & $|\cI|$ & 161 & 161 & 148 & 174 & 174 & 13 & 126 & \textbf{187} & \textbf{187} & 177 & 176 \\	 
			& time (s)  & 4.29 & 3.92 & 5.93 & 4.31 & 8.01 & 2.9 & 1.42 & 7.92 & 5.51 & 12.33 & 11.44\\ \hline 		
			
			\multirow{2}{*}{\begin{tabular}[c]{@{}c@{}}Union House\\ N = 520 \end{tabular}} & $|\cI|$ & 213 & 213 & 199 & 224 & 230 & 14 & 65 & 231 & 231 & \textbf{232} & 208 \\	 
			& time (s)  & 1.56 & 1.64 & 2.5 & 3.27 & 3.51 & 3.72 & 1.78 & 2.84 & 3.59 & 7.73 & 7.35\\ \hline 		
	
			\multirow{2}{*}{\begin{tabular}[c]{@{}c@{}}Old Classic Wing\\ N = 557 \end{tabular}} & $|\cI|$ & 198 & 208 & 126 & 209 & 210 & 52 & 147 & \textbf{216} & 206 & 210 & 197 \\	 
			& time (s)  & 1.85 & 1.47 & 2.57 & 3.32 & 3.96 & 2.77 & 1.47 & 5.29 & 7.57 & 9.06 & 10.23\\ \hline
		
		\multirow{2}{*}{\begin{tabular}[c]{@{}c@{}}Ball Hall\\ N = 534 \end{tabular}} & $|\cI|$ & 225 & 227 & 221 & 227 & 230 & 195 & 186 & \textbf{250} & \textbf{250} & 247 & 247 \\	 
		& time (s)   & 1.35 & 1.37 & 1.29 & 1.33 & 1.34 & 0.57 & 1.05 & 3.47 & 2.95 & 6.45 & 7.35\\ \hline 		
		
		\multirow{2}{*}{\begin{tabular}[c]{@{}c@{}}Building 64\\ N = 427 \end{tabular}} & $|\cI|$ & 123 & 128 & 100 & 135 & 133 & 73 & 82 & \textbf{142} & \textbf{142} & \textbf{142} & \textbf{142} \\	 
		& time (s)   & 3.27 & 2.56 & 10.11 & 3.63 & 5.93 & 1.17 & 0.99 & 6.95 & 7.54 & 10.07 & 9.05\\ \hline 		
		
		\multirow{2}{*}{\begin{tabular}[c]{@{}c@{}}Building 10\\ N = 525 \end{tabular}} & $|\cI|$ & 201 & 225 & 210 & 215 & 226 & 176 & 165 & \textbf{229} & \textbf{229} & 226 & 210 \\	 
		& time (s)  & 1.48 & 1.48 & 0.95 & 1.46 & 1.38 & 1.14 & 1.71 & 6.66 & 7.59 & 12.56 & 9.48 \\ \hline
		
		\multirow{2}{*}{\begin{tabular}[c]{@{}c@{}}Building 15\\ N = 596 \end{tabular}} & $|\cI|$ & 215 & 217 & 221 & 225 & 232 & 240 & 245 & \textbf{260} & \textbf{260} & \textbf{260} & \textbf{260} \\	 
		& time (s) & 1.94 & 1.82 & 1.78 & 1.65 & 1.67 & 1.62 & 1.17 & 5.39 & 4.56 & 9.31 & 11.08 \\ \hline
			
		\end{tabular}
	}
	\label{table:homography_results}
\end{table*}

%% file: affine_results.tex
\begin{table*}
	\centering
	\caption{Affinity estimation results (with geometric transfer error)}
	\resizebox{\linewidth}{!}{
		\begin{tabular}{|c|c|c|c|c|c|c|c|c|c|c|c|c|}
			\hline
			\multicolumn{2}{|c|}{Methods} & RS & PS & GMLE & LORS & LORS1 & $\ell_1$ & $\ell_\infty$ & EP-RS & EP-$\ell_\infty$ & AM-RS & AM-$\ell_\infty$ \\ \hline		
			\multirow{2}{*}{\begin{tabular}[c]{@{}c@{}}Bikes \\ N = 518\end{tabular}} & $|\cI|$  & 410 & 410 & 410 & 411 & 410 & 412 & 415 & \textbf{421} & \textbf{421} & 417 & 417 	\\	 
			& time (s) & 5.94 & 5.86 & 5.6 & 8.23 & 13.42 & 4.52 & 0.97 & 15.21 & 7.76 & 10.42 & 5.65 \\ \hline   
			
			\multirow{2}{*}{\begin{tabular}[c]{@{}c@{}} Tree \\ N = 465\end{tabular}} & $|\cI|$ & 286 & 288 & 289 & 287 & 286 & 301 & 278 & \textbf{311} & \textbf{311} & 305 & 307 	\\	 
			& time (s)& 5.94 & 5.86 & 5.6 & 8.23 & 13.42 & 4.52 & 0.97 & 15.21 & 7.76 & 10.42 & 5.65 \\ \hline   
			
			\multirow{2}{*}{\begin{tabular}[c]{@{}c@{}}Boat\\ N = 402 \end{tabular}} & $|\cI|$ & 308 & 311 & 304 & 310 & 308 & 330 & 330 & \textbf{340} & \textbf{340} & 325 & 330 \\	 
			& time (s)  & 5.61 & 5.63 & 5.31 & 6.62 & 10.91 & 2.46 & 0.88 & 10.34 & 5.59 & 10.12 & 5.05\\ \hline 		
		
			\multirow{2}{*}{\begin{tabular}[c]{@{}c@{}}Graff\\ N = 331 \end{tabular}} & $|\cI|$ & 140 & 141 & 142 & 141 & 140 & 304 & 308 & \textbf{313} & \textbf{313} & 308 & 308 \\	 
			& time (s)   & 4.95 & 4.7 & 4.32 & 5.91 & 9.34 & 1.39 & 0.39 & 10.82 & 6.26 & 17.18 & 11.7\\ \hline 		
			
			\multirow{2}{*}{\begin{tabular}[c]{@{}c@{}} Bark\\ N = 219 \end{tabular}} & $|\cI|$  & 194 & 195 & 195 & 194 & 194 & 200 & \textbf{203} & \textbf{203} & \textbf{203} & 202 & \textbf{203} \\	 
			& time (s)   & 3.01 & 3.06 & 3.41 & 3.42 & 5.61 & 0.32 & 0.32 & 3.86 & 1.17 & 14.21 & 14.49\\ \hline
			
			\multirow{2}{*}{\begin{tabular}[c]{@{}c@{}} Building 143\\ N = 537 \end{tabular}} & $|\cI|$  & 94 & 93 & 91 & 99 & 94 & 338 & 331 & 342 & 342 & \textbf{349} & 347 \\	 
			& time (s)  & 7.97 & 8.19 & 8.02 & 9.52 & 15.41 & 5.62 & 2.55 & 16.6 & 10.28 & 34.77 & 33.12\\ \hline
			
			\multirow{2}{*}{\begin{tabular}[c]{@{}c@{}} Building 152\\ N = 469 \end{tabular}} & $|\cI|$  & 198 & 192 & 173 & 211 & 198 & 221 & 228 & \textbf{281} & \textbf{281} & 277 & 277 \\	 
			& time (s) & 6 & 6 & 5.71 & 7.71 & 11.67 & 3.16 & 1.71 & 12.41 & 7.75 & 28.2 & 24.03\\ \hline  		  		
			
			\multirow{2}{*}{\begin{tabular}[c]{@{}c@{}} Building 163\\ N = 617 \end{tabular}} & $|\cI|$  & 306 & 308 & 303 & 307 & 306 & 402 & 399 & \textbf{437} & \textbf{437} & 431 & 430 \\	 
			& time (s)  & 7.85 & 7.82 & 7.58 & 8.93 & 15.3 & 8.06 & 3.37 & 16.93 & 11.64 & 21.93 & 17.04\\ \hline  		  		
			
			\multirow{2}{*}{\begin{tabular}[c]{@{}c@{}} Building 170\\ N = 707 \end{tabular}} & $|\cI|$  & 315 & 311 & 311 & 318 & 315 & 455 & 412 & \textbf{538} & \textbf{538} & 524 & 525 \\	 
			& time (s) & 9.48 & 9.46 & 9.25 & 11.65 & 18.72 & 11.24 & 2.18 & 31.66 & 23.73 & 61.65 & 57.71\\ \hline  		  		
			
			\multirow{2}{*}{\begin{tabular}[c]{@{}c@{}} Building 174\\ N = 580 \end{tabular}} & $|\cI|$ & 339 & 338 & 339 & 341 & 339 & 334 & 312 & 369 & 369 & \textbf{375} & 374 \\	 
			& time (s) & 7.8 & 7.73 & 7.4 & 9.78 & 15.13 & 5.94 & 1.89 & 17.92 & 11.77 & 50.48 & 38.63\\ \hline  		  		
		\end{tabular}
	}
	\label{table:affine}
\end{table*}

%% file: triangulation_results.tex
\begin{table*}[ht]
  \centering
  \caption{Triangulation results (with geometric transfer error)}  
  \begin{tabular}{|c|c|c|c|c|c|c|c|c|c|c|}
    \hline
	\multicolumn{2}{|c|}{Methods} & RS & LORS & LORS1 & $\ell_1$   & $\ell_{\infty}$ & EP-RS  & EP-$\ell_{\infty}$ & AM-RS &  AM-$\ell_{\infty}$ \\ \hline
	
	\multirow{2}{*}{\begin{tabular}[c]{@{}c@{}}Point 1\\ N = 167\end{tabular}}    & $|\cI|$        & 95 & 96 & 95 & 96 & 81 & \textbf{97} & \textbf{97} & \textbf{97} & 96    \\ %\cline{2-12} 
	& time (s)  & 0.22 & 0.47 & 0.38 & 0.8 & 2.04 & 3.69 & 4.34 & 1.81 & 3.63     \\ \hline
	
	\multirow{2}{*}{\begin{tabular}[c]{@{}c@{}}Point 3\\ N = 145\end{tabular}}    & $|\cI|$       & 82 & 84 & 82 & 79 & 53 & 85 & 84 & \textbf{86} & 77    \\ %\cline{2-12} 
	& time (s)  & 0.15 & 0.32 & 0.29 & 0.16 & 1.16 & 1.92 & 2.97 & 2.04 & 2.5     \\ \hline
	
	\multirow{2}{*}{\begin{tabular}[c]{@{}c@{}}Point 9\\ N = 135\end{tabular}}    & $|\cI|$       & 49 & 51 & 49 & 30 & 38 & \textbf{52} & 49 & \textbf{52} & 47    \\ %\cline{2-12} 
	& time (s) & 0.16 & 0.39 & 0.28 & 0.14 & 0.84 & 2.37 & 3.08 & 1.42 & 4.37     \\ \hline
	
	\multirow{2}{*}{\begin{tabular}[c]{@{}c@{}}Point 15\\ N = 140\end{tabular}}    & $|\cI|$       & 50 & 53 & 50 & 43 & 38 & 53 & 46 & \textbf{55} & 41    \\ %\cline{2-12} 
	& time (s) & 0.15 & 0.36 & 0.27 & 0.24 & 1.14 & 2.63 & 3.52 & 1.4 & 4.16    \\ \hline
	
	\multirow{2}{*}{\begin{tabular}[c]{@{}c@{}}Point 24\\ N = 155\end{tabular}}    & $|\cI|$       & 110 & 113 & 110 & 113 & 111 & 113 & 113 & \textbf{114} & \textbf{114}    \\ %\cline{2-12} 
	& time (s)  & 0.17 & 0.34 & 0.31 & 0.13 & 0.44 & 2.24 & 2.59 & 1.67 & 1.93    \\ \hline
	
	\multirow{2}{*}{\begin{tabular}[c]{@{}c@{}}Point 72\\ N = 104\end{tabular}}    & $|\cI|$       & 38 & 39 & 38 & 37 & 35 & \textbf{41} & \textbf{41} & \textbf{41} & 39    \\ %\cline{2-12} 
	& time (s)  & 0.12 & 0.29 & 0.21 & 0.08 & 0.54 & 1.15 & 1.53 & 1.12 & 1.57   \\ \hline
	
	\multirow{2}{*}{\begin{tabular}[c]{@{}c@{}}Point 82\\ N = 118\end{tabular}}    & $|\cI|$      & 56 & 58 & 56 & 55 & 48 & 59 & 59 & \textbf{60} & 55   \\ %\cline{2-12} 
	& time (s)  & 0.13 & 0.33 & 0.23 & 0.09 & 0.4 & 1.43 & 1.82 & 1.22 & 1.48   \\ \hline
	
	\multirow{2}{*}{\begin{tabular}[c]{@{}c@{}}Point 192\\ N = 123\end{tabular}}    & $|\cI|$      & 89 & 90 & 89 & 92 & 87 & 91 & 91 & \textbf{93} & 92   \\ %\cline{2-12} 
	& time (s)  & 0.14 & 0.27 & 0.26 & 0.09 & 0.39 & 1.15 & 1.41 & 1.27 & 1.51   \\ \hline       
	
	\multirow{2}{*}{\begin{tabular}[c]{@{}c@{}}Point 193\\ N = 132\end{tabular}}    & $|\cI|$     & 113 & 114 & 113 & 111 & 113 & \textbf{117} & \textbf{117} & 116 & \textbf{117}   \\ %\cline{2-12} 
	& time (s)  & 0.14 & 0.28 & 0.26 & 0.09 & 0.45 & 0.99 & 1.28 & 1.29 & 1.67   \\ \hline      
	
	\multirow{2}{*}{\begin{tabular}[c]{@{}c@{}}Point 249\\ N = 124\end{tabular}}    & $|\cI|$     & 93 & \textbf{94} & 93 & 93 & 90 & \textbf{94} & 92 & \textbf{94} & 92   \\ %\cline{2-12} 
	& time (s) & 0.13 & 0.27 & 0.24 & 0.1 & 0.36 & 1.59 & 1.84 & 1.31 & 1.61   \\ \hline      
  \end{tabular}
   \label{table:triangulation_results}
\end{table*}

%% file: triangulation_results_summ.tex
\begin{table}[ht]
  \centering
  \caption{Total inliers and runtime of triangulation for 11595 selected points with more than 10 views}  
  \begin{tabular}{|c|c|c|}
  	\hline
  	Methods & Total inliers & Time (minutes) \\ \hline
  	RS      &   91888      &    12.10  \\ \hline
  	LORS    &    94387     &   23.09   \\ \hline
  	LORS1   &    91555      & 20.84     \\ \hline
  	$\ell_1$ &   40669       & 11.16     \\ \hline
  	$\ell_{\infty}$  &   43869       &   45.18   \\ \hline
  	EP-RS &   \textbf{99232}      & 49.52     \\ \hline
  	EP-$\ell_{\infty}$ &   59996       & 71.86\\ \hline
  	AM-RS &   97453       &  86.14    \\ \hline
  	AM-$\ell_\infty$ &      49760    &   125.74   \\ \hline
  \end{tabular}
   \label{table:triangulation_results_all}
\end{table}